\documentclass{article} % For LaTeX2e

\usepackage{natbib}
\setcitestyle{numbers, compress}
\usepackage[preprint]{neurips}

\usepackage{amsmath,amsfonts,bm}

\def\1{\bm{1}}

\DeclareMathAlphabet{\mathsfit}{\encodingdefault}{\sfdefault}{m}{sl}
\SetMathAlphabet{\mathsfit}{bold}{\encodingdefault}{\sfdefault}{bx}{n}

\DeclareMathOperator{\Tr}{Tr}

\newcommand{\bv}{{v}}

\newcommand{\bx}{{x}}

\newcommand{\bF}{{F}}

\newcommand{\bI}{{I}}

\newcommand{\btheta}{{\theta}}

\newcommand{\bxi}{{\xi}}

\newcommand{\cG}{\mathcal{G}}

\newcommand{\cL}{\mathcal{L}}

\newcommand{\cN}{\mathcal{N}}
\newcommand{\cO}{\mathcal{O}}

\newcommand{\cQ}{\mathcal{Q}}

\newcommand{\cV}{\mathcal{V}}
\newcommand{\cW}{\mathcal{W}}

\newcommand{\bbE}{\mathbb{E}}

\newcommand{\bbI}{\mathbb{I}}

\newcommand{\bbP}{\mathbb{P}}

\newcommand{\bbR}{\mathbb{R}}
\newcommand{\bbS}{\mathbb{S}}

\newcommand{\pll}{\kern 0.56em/\kern -0.8em /\kern 0.56em}

\newcommand{\norm}[1]{\ensuremath{\left\| #1 \right\|}}

\newcommand{\<}{\left\langle}
\renewcommand{\>}{\right\rangle}

\usepackage[utf8]{inputenc} % allow utf-8 input
\usepackage[T1]{fontenc}    % use 8-bit T1 fonts
\usepackage{titletoc}
\usepackage[pagebackref]{hyperref}       % hyperlinks
\usepackage{url}            % simple URL typesetting
\usepackage{booktabs}       % professional-quality tables
\usepackage{amsfonts}       % blackboard math symbols
\usepackage{nicefrac}       % compact symbols for 1/2, etc.
\usepackage{microtype}      % microtypography
\usepackage{xcolor}         % colors
\usepackage{wrapfig}

\usepackage{enumitem}

\newcommand{\diag}{{\rm diag}}

\newcommand{\srk}{{\rm srk}}
\newcommand{\rF}{\mathrm{F}}
\newcommand{\sbracket}[1]{\ensuremath{\left( #1 \right)}}
\newcommand{\mbracket}[1]{\ensuremath{\left[ #1 \right]}}
\newcommand{\bbracket}[1]{\ensuremath{\left\{ #1 \right\}}}
\renewcommand{\>}{\right\rangle}
\newcommand{\EE}{\mathbb{E}}

\newcommand{\RR}{\mathbb{R}}
\renewcommand{\wp}{{\em w.p.~}}
\newcommand{\tr}{\mathrm{tr}}

\usepackage{wrapfig}

\usepackage[ruled,vlined]{algorithm2e}

\usepackage{todonotes}

\hypersetup{colorlinks=true,linkcolor=red!70!black,linktocpage=false,citebordercolor=blue!70!black,citecolor=blue!70!black,anchorcolor=blue!70!black}

\usepackage{graphicx,subfig}
\usepackage{url}
\usepackage{multirow}
\usepackage{amsmath}
\usepackage{amssymb}
\usepackage{mathtools}
\usepackage{amsthm}
\usepackage{enumitem}

\author{Mingze Wang \\
School of Mathematical Sciences\\
Peking University\\
\texttt{mingzewang@stu.pku.edu.cn} 
\And
Lei Wu~\dag \\
School of Mathematical Sciences\\ 
Center for Machine Learning Research\\
Peking University\\
\texttt{leiwu@math.pku.edu.cn} 
\vspace*{-1em}
}

\theoremstyle{plain}
\newtheorem{theorem}{Theorem}[section]
\newtheorem{proposition}[theorem]{Proposition}
\newtheorem{lemma}[theorem]{Lemma}

\theoremstyle{definition}
\newtheorem{definition}[theorem]{Definition}

\newtheorem{assumption}[theorem]{Assumption}

\newtheorem*{main result}{Main Theorem}

\allowdisplaybreaks[4]

\title{A Theoretical Analysis of Noise Geometry in Stochastic Gradient Descent}

\begin{document}

\maketitle

\begin{abstract}
In this paper, we provide a theoretical study of noise geometry for minibatch  stochastic gradient descent (SGD), a phenomenon where noise aligns favorably   with the  geometry of local landscape.
We propose two  metrics, derived from analyzing how noise influences the loss and subspace projection dynamics, to quantify the alignment strength.  We show that for (over-parameterized) linear models and two-layer nonlinear networks, when measured by these metrics, the  alignment  can be provably guaranteed  under conditions  independent of the degree of over-parameterization.
To showcase the utility of our noise geometry characterizations, we present a refined analysis of the mechanism by which SGD  escapes  from sharp minima. We reveal that unlike gradient descent (GD), which escapes along the sharpest directions, SGD tends to escape from flatter directions and cyclical learning rates can exploit this SGD characteristic to navigate more effectively towards flatter regions. Lastly, extensive experiments are provided to support our theoretical findings.

\end{abstract}

\section{Introduction}
% \vspace*{.5em}
Minibatch SGD\footnote{\dag\ Corresponding Author.}, along with its variants, now stands as the de facto optimizers in  modern machine learning (ML) \citep{goodfellow2016deep}. 
Unlike full-batch GD, SGD uses only mini-batches of data  in each iteration, which injects noise into the training process.  Considering a loss objective $\cL:\RR^p\mapsto\RR$, the  SGD iteration with learning rate $\eta$  can be written as
\begin{align}\label{eqn: a0}
     \theta_{t+1}=\theta_t-\eta\sbracket{\nabla \cL(\theta_t)+  \xi_t},
\end{align}
where $\xi_t=\xi(\theta_t)$ denotes the injected noise that satisfies $\EE[\xi_t]=0$ and $\EE[\xi_t\xi_t^\top]=\Sigma(\theta_t)$, with $\Sigma(\cdot)$ denoting the covariance matrix of SGD noise. Importantly, the injected noise can significantly alter the optimizer's properties, particularly in aspects like optimization convergence 
\citep{hazan2016introduction,thomas2020interplay,wojtowytsch2023stochastic} 
and implicit regularization \citep{zhang2017understanding,keskar2016large,wu2017towards,wu2018sgd}.

To illustrate this, we begin by examining the noise's influence on loss dynamics. Assuming $\cL(\cdot)$ is twice differentiable and letting $H(\btheta)=\nabla^2\cL(\theta)$,  then the one-step loss update of SGD is given by
\begin{equation}\label{eqn: x2}
\EE[\cL(\theta_{t+1}] = \EE[\cQ(\theta_t)]  + \frac{\eta^2}{2} \EE[\gamma(\theta_t)] + o(\eta^2),
\end{equation}
where ${\cQ(\btheta)=\cL(\btheta)-\eta\|\nabla \cL(\btheta)\|^2+\frac{\eta^2}{2}\nabla \cL(\btheta)^\top H(\btheta)\nabla \cL(\btheta)}$ denotes the contribution of GD component and $\gamma(\btheta_t)=\EE[\bxi_t^\top H(\btheta)\bxi_t]$ represents  the noise contribution.  
In optimization literature, one often makes the {\em bounded variance assumption}: $\EE[\|\bxi_t\|^2]\leq\sigma^2$, with $\sigma$ being a fixed constant.
By additionally assuming $\|H(\btheta_t)\|_2\leq L$, it follows that 
% \wmz{citation}
$
\gamma(\theta_t) \leq L\EE[\|  \xi_t\|^2]\leq L\sigma^2.
$
However, this estimates could substantially  overestimate $\gamma(\theta_t)$ as it overlooks the fact:  for typical ML models, both the magnitude and shape of SGD noise are state-dependent.
 % \citep{bottou2018optimization,zhu2019anisotropic}.

\vspace{.5em}
\textbf{Noise magnitude.} \citet{bottou2018optimization} introduced the affine variance assumption: $\EE[\|  \xi_t\|^2]\leq \sigma_0^2+\sigma_1^2\|\nabla \cL(\theta_t)\|^2$, incorporating the state-dependence of noise magnitude.  
This  assumption has been adopted in \cite{faw2022power} to justify the superiority of adaptive stochastic gradient methods.  Additionally, a series of studies including  \cite{feng2021inverse,mori2021logarithmic,wojtowytsch2021stochastic,liu2021noise} have shown, for regression with square loss, the magnitude of SGD noise can be bounded by the loss value: $\EE[\|  \xi_t\|^2]\leq \sigma^2\cL(\theta_t)$. This reveals that the noise magnitude diminishes to zero at global minima. Leveraging this property, some works
\citep{bassily2018exponential,fang2020fast,wojtowytsch2021stochastic,liu2023aiming} showed that  SGD can  achieve  convergence with a constant learning rate in the interpolation regime. This contrasts starkly with the bounded variance assumption, where SGD requires a decaying learning rate to converge \citep{hazan2016introduction}.

\textbf{Noise shape.}
It is important to note that the noise contribution in \eqref{eqn: x2} can be expressed as 
\begin{equation}\label{eqn: noise-influence}
    \gamma(\theta_t)=\EE[  \xi^\top_t H(\theta_t)  \xi_t]
    =\EE[  \xi_t  \xi^\top_tH(\theta_t)]=\tr(\Sigma(\theta_t)H(\theta_t)).
\end{equation}
This underscores the significance of the noise shape, highlighting the necessity of examining the relationship between the noise covariance and the local Hessian matrix.
In particular, \cite{zhu2019anisotropic,wu2020noisy,xie2020diffusion} have empirically demonstrated that  $\Sigma(\btheta)$ is highly {\em anisotropic}  and bears resemblance to $H(\btheta)$ to a certain degree for various ML models. However, it is still unclear how to quantitatively characterize this resemblance.
% Let $H(\theta_t)=\sum_{j=1}^p \lambda_j u_j u_j^\top$ be the eigen-decomposition of $H(\theta_t)$, where we omit the dependence of $\theta_t$ for simplicity. 
% Then, we have 
% $
% 	\gamma(\theta_t) = \sum_{j=1}^p \lambda_j \EE[(\xi_t^\top u_j)^2],
% $
% suggesting that $\gamma(\theta_t)$ can be viewed as a weighted average of noise energy along different eigen-directions of Hessian. 
% Therefore, in order to estimate the noise contribution, we need to explore {\em the relation between }.

% Consequently, they suggested a Hessian-based approximation of the noise covariance: $\Sigma(\btheta)\approx\eps^2 H(\btheta)$, where $\eps$ is a small constant denoting the noise magnitude. 

Beyond loss convergence, it has been demonstrated that the aforementioned  noise geometry is also crucial in shaping the superior generalization properties of SGD \citep{zhu2019anisotropic,wu2020noisy,li2021validity,wu2022does,haochen2021shape} and general dynamical behaviors \citep{feng2021inverse,ziyin2021minibatch,wu2018sgd,thomas2020interplay}. Therefore, it is important to establish quantitative descriptions of the above noise geometry.

In the context of regression  with square loss, by incorporating the state dependence of  both noise magnitude and shape,
\cite{mori2021logarithmic} proposed a heuristic  approximation of the noise covariance: 
\begin{equation}\label{eqn: decoupl-x1}
\Sigma(\btheta)\approx 2\cL(\btheta) G(\btheta),
\end{equation}
where $G(\theta)$ denotes the empirical Fisher matrix and $G(\theta)\approx H(\btheta)$ in low-loss region (see Section \ref{sec: preliminary} for details). This approximation suggests an {\em intriguing alignment between SGD noise and local landscape}: 1) the noise magnitude is proportional to the loss value; 2) the noise energy tends to concentrate more along sharp directions than flat directions. The latter point can be deduced as follows: for a fixed direction $u$, the noise energy along $v$ satisfies 
\begin{equation}\label{eqn: decoupl-subspace}
\EE[(\xi(\theta)^\top v)^2] = v^\top \Sigma(\theta) v \approx 2\cL(\theta) v^\top G(\theta) v,
\end{equation}
where the second step uses \eqref{eqn: decoupl-x1} and  $u^\top G(\theta)u$ is roughly the curvature of local landscape along the direction $u$. However,   it should be noted that this heuristic approximation  can be neither theoretically justified nor empirically verified. 

\subsection{Our contribution}

The goal of this work is to advance beyond  heuristic approximations of SGD noise by  providing theoretical characterizations. We highlight that the inaccuracy of approximation \eqref{eqn: decoupl-x1} stems from its attempt to establish a full description of the entire noise covariance. This is often unnecessary since, in most scenarios,  the focus is only on some specific low-dimensional quantities instead of the entire high-dimensional trajectory. For instance, when analyzing the loss dynamics,  it suffices to characterize  $\gamma(\theta_t)\in\RR$ as per Eq.~\eqref{eqn: x2} rather than  $\Sigma(\theta_t)\in\RR^{p\times p}$.

Adopting this perspective, we demonstrate that the alignment implied by  approximation \eqref{eqn: decoupl-x1}, while not accurate, is still valid when assessed using less stringent metrics. 
In addition, we present an illustrative example, showcasing the utility of our noise geometry quantifications in analyzing SGD's dynamical behavior. 
Specifically, our key findings are outlined as follows. 

% Our theoretical analysis considers linear models, over-parameterized linear models (OLMs), and two-layer nonlinear networks (without bias). 
% \wmz{It seems that this paragraph only focuses on Sec 3, and it does not summarize our {\bf roadmap}.}
% % We establish theoretical guarantees  for simplified models such as . Specifically, our theoretical findings are summarized as follows.
% % Additionally, we also provide  experiments on both synthetic and real-world settings to corroborate our theoretical findings.

\begin{itemize}[leftmargin=2em]
    \item 
    We begin by quantifying the noise geometry through examining its influence on loss dynamics.  According to Eq.~\eqref{eqn: noise-influence},  when approximation \eqref{eqn: decoupl-x1} holds, $\gamma(\theta)\approx 2\cL(\theta)\|G(\theta)\|^2_F=:\bar{\gamma}(\theta)$ in regions of low loss. This motivates us to  measure the alignment strength using the ratio $\mu(\theta)=\gamma(\theta)/\bar{\gamma}(\theta)$. We prove that, in both (over-parameterized) linear models and two-layer nonlinear networks without bias,  $\mu(\theta)$ is close to $1$ across the entire parameter space, a result that intriguingly holds true regardless of the degree of over-parameterization.
    
     % We offer a comprehensive investigation of  how factors, such as the number of model parameters, sample size,  and input dimension, influence the alignment strength.  
    % We show that for linear regression, provided $d_{\mathrm{eff}}\gtrsim \log n$, the  alignment strength is both lower and upper bounded by absolute constants for OLMs and two-layer networks. This condition accommodates the important regimes like $n\sim \log (d_{\mathrm{eff}})$ (for sparse recovery) and $n\sim d_{\mathrm{eff}}$ (the proportional scaling).  \wl{to be updated accordingly}

	\item In our next analysis, we focus on the alignment of SGD noise with local geometry along fixed directions, investigating if {\em the noise energy along a given direction is proportional to the curvature in that direction}. This type of alignment can be useful in describing the SGD dynamics in  subspaces, such as the top principal components (see Section \ref{section: theory, strong align} for details). According to Eq.~\eqref{eqn: decoupl-subspace}, we  define the metric $g(\theta,v)=v^\top\Sigma(\theta)v/(2\cL(\theta)v^\top G(\theta)v)$ to quantify the strength of directional alignment, where  $v$ denotes the direction of interest. We establish that for (over-parameterized) linear models, this directional alignment also holds regardless of the degree of over-parameterization.

\item  Lastly, to illustrate the utility of our characterization of noise geometry, we provide a fine-grained analysis of the mechanism by which  SGD escapes from sharp minima. We show that \textit{the escape direction of SGD exhibits significant components along flat directions}. This stands in stark contrast to GD, which escapes from minima only along the sharpest direction. We also discuss an implication of this unique escape property: cyclical learning rate \citep{smith2017cyclical,loshchilovsgdr} can leverage it  to help SGD locate flat minima more effectively.
\end{itemize}

To validate our theoretical findings, we have provided both small-scale and large-scale experiments, including the classification of  CIFAR-10 dataset using VGG nets and ResNets. Overall, we not only establish quantitative  descriptions for the alignment between SGD noise and the local geometry but also shed light on  how the unique noise geometry helps SGD navigate the loss landscape.

\subsection{Other related work}
% \vspace*{-.5em}

{\bf The noise geometry.} \citet{ziyin2021minibatch} conducted an analysis of the  noise structure of online SGD for under-parameterized linear regression. \citet{pesme2021implicit} studied the implicit bias of SGD noise for diagonal linear networks.
\citet{haochen2021shape,damian2021label,li2021happens} showed that the noise covariance of label-noise SGD is $\varepsilon^2 G(\theta)$, where $\varepsilon$ denotes the size of label noise. 
% We instead show that even without label noise, the noise covariance possesses a similar structure in certain weak senses. 
Works such as \citet{simsekli2019tail,zhou2020towards} argued that the magnitude of SGD noise  is heavy-tailed but not considered the noise shape. 
Lastly, we remark that
aligning with our first part analysis, \citet{wu2022does}  developed theoretical underpinnings for the noise geometry in terms of its contribution in loss dynamics. However, their analysis  is restricted to OLMs under the infinite-data regime. In contrast, we offer an investigation of the finite-sample effects and extend substantially beyond the scope of \citet{wu2022does}. 
% and the observation that the noise magnitude can be bounded by the loss value.

% \cite{wu2022does} theoretically analyzes the alignment between Fisher matrix and noise covariance matrix for both OLMs and random feature models but for the infinite-sample regime. In contrast, we provide a comprehensive analysis of finite-sample regimes and moreover, consider two-layer nonlinear networks.

%
% As an application,  (please see the next paragraph).
% However, most studies still treat SGD as SDE, which is good modeling only in finite time and small learning rate (LR) regime \citep{li2017stochastic,li2019stochastic}. 
% The relevance of SDE modeling for understanding SGD with large LR remains unclear \citep{li2021validity}. 
% Furthermore, most works only focus on escape rates from sharp minima rather than exploring escaping trajectories or directions.
% In this work, we fill these gaps by providing a fine-grained analysis of the noise geometry of mini-batch SGD. 

{\bf Escape from minima and saddle points.} The phenomenon of SGD escaping from sharp minima exponentially fast was initially studied in \citet{zhu2019anisotropic,wu2018sgd}. This provides an explanation of the famous ``flat minima hypothesis'' \citep{hochreiter1997flat,keskar2016large,wu2023implicit}---one of the most important observations in explaining the implicit regularization of SGD. 
However, existing analyses of the escape phenomenon  have primarily focused on the escape rate \citep{wu2018sgd,zhu2019anisotropic,xie2020diffusion,mori2021logarithmic,ziyin2021minibatch}. In contrast, we extends this focus by providing an analysis of escape direction, which is enabled by our refined description of the noise geometry. 
\citet{kleinberg2018alternative} introduced an alternative perspective, positing that SGD circumvents local minima by navigating an effective loss landscape that results from the convolution of the original landscape with SGD noise. In this context, our noise geometry results can be beneficial in understanding the effective loss landscape. In addition, prior works like \citep{daneshmand2018escaping,xie2022adaptive} has illustrated that the alignment of SGD noise with local geometry facilitates the rapid saddle-point escape of SGD. Our work offers theoretical support for the alignment assumptions in these studies.

% Such an analysis is also potentially valuable for understanding the implicit regularization effect of employing cyclical learning rates \citep{smith2017cyclical,loshchilovsgdr,huangsnapshot}. Specifically, we refer to Figure 2 in \cite{huangsnapshot}, which visually demonstrates that SGD repeatedly escapes from previous minima when learning rate is increased.

% In the study of SGD convergence in linear regression,  it is often assumed that $\Sigma(\theta) \preccurlyeq \sigma^2 G$, and the convergence rate is roughly $\tr(\Sigma(\theta^*)H^{-1})/n$ \citep{flammarion2015averaging}. Our directional alignment result implies that $\sigma^2\sim 2L(\theta)$.

% \paragraph*{Local convergence}
% \paragraph*{Convergence}

% \cite{xie2020diffusion,mori2021logarithmic} study this issue using the classical diffusion-based framework \cite{gardiner2009stochastic} based on the continuous-time stochastic differential equation (SDE). 

% \vspace*{-.5em}
\section{Preliminaries}
\label{sec: preliminary}
% \vspace*{-.5em}

{\bf Notation.}
% We use bold letters for vectors and lowercase letters for scalars, e.g. $\boldsymbol{x}=(x_1,\cdots,x_d)^\top$.
We use $\left<\cdot,\cdot\right>$ for the Euclidean inner product  and $\left\|\cdot\right\|$ the $\ell^2$ norm of a vector or the spectral norm of a matrix. For any positive integer $k$, let $[k]=\{1,\cdots,k\}$. 
Denote by $\cN( \mu,S)$ the Gaussian distribution with mean $ \mu$ and covariance matrix $S$.
% , while we define $\cU(\Omega)$ as the uniform distribution on a set $\Omega$. 
% Given a matrix $A$, we refer to its eigenvalues in a decreasing order as $\{\lambda_j(A)\}_{j\geq 1}$. 
When $A$ is positive semidefinite,  we use $\srk(A):=\tr(A)/\norm{A}$ to denote  the effective rank of $A$. 
We use $a \lesssim b$ to mean there exist an  an absolute  constant $C>0$ such that $a \leq Cb$ and $a \gtrsim b$ is defined analogously. We write $a \sim b$ if both $a\lesssim b$ and $a\gtrsim b$ hold.

{\bf Problem Setup.}
Let $\{(\bx_i,y_i)\}_{i=1}^n$ with $x_i\in\RR^d$ and $y_i\in\RR$ for $i\in [n]$ be the training set and 
 $f(\cdot;\btheta):\bbR^d\to\bbR$ be the model parameterized by $\btheta\in\bbR^p$. Let $\ell_i(\btheta)=\frac{1}{2}\sbracket{f(\bx_i;\btheta)-y_i}^2$  and $\cL(\btheta)=\frac{1}{n}\sum_{i=1}^n\ell_i(\btheta)$ be the empirical loss.  In theoretical analysis, we make the following input  assumption:
 \begin{assumption}\label{assumption: input}
Suppose that $\bx_1,\bx_2,\dots,\bx_n$ are {\em i.i.d} samples drawn from $\cN(0,S)$ with $S\in\RR^{d\times d}$ denoting the input covariance matrix. We  use $
d_{\mathrm{eff}}:=\min\{\srk(S),\srk(S^2)\}
$ to denote the effective input dimension. 
 \end{assumption}
 % We remark that  numerical experiments on non-Gaussian data are also provided to validate our theoretical findings.
% Our  analysis also needs the following definition of the {\em effective dimension of inputs}:
% \begin{definition}[Effective input dimension]\label{eq: effective-dim}

% \end{definition}
% It is obvious that when $S$ is isotropic, we have $d_{\mathrm{eff}}=d$. 

In the above setup, the Hessian matrix of the empirical loss is given by
\begin{equation}\label{eqn: Hessian}
H(\theta)=G(\theta)+\frac{1}{n}\sum_{i=1}^n\sbracket{f( x_i;\theta)-y_i}\nabla^2 f(  x_i;\theta),
\end{equation}
where 
$
G(\theta)=\frac{1}{n}\sum_{i=1}^n\nabla f(  x_i;\theta)\nabla f(  x_i;\theta)^\top
$
 is the empirical Fisher matrix. Eq.~\eqref{eqn: Hessian} implies that when the fit errors are negligible, we have $G(\btheta)\approx H(\btheta)$ and in particular, at global minima $\btheta^*$,  $ H(\btheta^*)=G(\btheta^*)$. Additionally, for linear regression $f(\bx;\btheta)=\btheta^{\top}\bx$, $ H(\btheta)= G(\btheta)\equiv\frac{1}{n}\sum_{i=1}^n\bx_i\bx_i^\top$. Therefore, in this paper, we focus on examining the alignment between noise covariance $\Sigma(\theta)$ and the empirical Fisher matrix $G(\btheta)$, rather directly analyzing the Hessian.

{\bf SGD and noise covariance.} To minimize $\cL(\cdot)$, SGD with a batch size of $1$ updates as $\theta_{t+1}=\theta_t-\eta \nabla\ell_{i_t}(\theta_t)$ with $(i_t)_{t\geq 1}$ being {\em i.i.d} samples uniformly drawn from $[n]$. In this case, the noise covariance is given by
$
\Sigma_0(\theta)=\Sigma_1(\btheta)-\Sigma_2(\btheta)
$
with
\begin{equation}
    \begin{aligned}
        \Sigma_1(\btheta)&=\frac{1}{n}\sum_{i=1}^n \nabla\ell_i(\btheta)\nabla\ell_i(\btheta)^\top,
        \\\Sigma_2(\btheta)&=\nabla \cL(\btheta) \nabla \cL(\btheta)^\top.
    \end{aligned}
\end{equation}
Since the relationship between $\Sigma_2(\btheta)$ with $\nabla \cL(\theta)$ is evident, the remaining task is to characterize the geometry of  $\Sigma_1(\btheta)$. Therefore, in the subsequent analysis, we shall refer $\Sigma(\btheta)$ as  $\Sigma_1(\btheta)$ for simplicity. 

Noting $\nabla \ell_i(\theta)=(f( x_i;\theta)-y_i)\nabla f(x_i;\theta)$, it follows that
$
	\Sigma(\theta) = \frac{1}{n}\sum_{i=1}^n (f( x_i;\theta)-y_i)^2\nabla f(x_i;\theta) \nabla f(x_i;\theta)^\top.
$
Then, the heuristic Hessian-based approximation \eqref{eqn: decoupl-x1} follows by
assuming  fitting errors $\{f( x_i;\theta)-y_i\}_{i=1}^n$ and model gradients $\{\nabla f(x_i;\theta)\}_{i=1}^n$ are decoupled \citep{mori2021logarithmic}.
%
% As discussed previously, this approximation cannot be true in general. However, we will show that it does hold in some weak sense, which can be theoretically justified. 
% In general, this decoupled approximation cannot be true but it reveals 1) the noise magnitude is proportional to the loss value; 2) the noise covariance aligns with the Fisher matrix. This approximation is intuitive and helpful for understanding, but it cannot be accurate in general.

{\bf Over-parameterized linear models (OLMs).} 
An OLM is defined as $f(\bx;\btheta)=F(\btheta)^{\top}\bx$, where $F:\bbR^p\to\bbR^d$ denotes a general re-parameterization map. Although $f(\cdot;\btheta)$ only represents linear functions, the associated loss landscape can be highly non-convex. Some typical examples include (i)  linear model $F(w)=w$; (ii) diagonal linear network: $F(\btheta)=(\alpha_1^2-\beta_1^2,\dots,\alpha_d^2-\beta_d^2)^\top$; and (iii) linear network: $F(\btheta) =W_1W_2\cdots W_L$. 
Notably, OLMs have been widely used to analyze the optimization and implicit bias of SGD \citep{arora2019implicit,woodworth2020kernel,pesme2021implicit,haochen2021shape,azulay2021implicit,li2021happens}.

\section{Loss alignment}\label{sec: average}

By Eq.~\eqref{eqn: x2}, to assess how the noise geometry affects the loss dynamics, we only  need  to estimate
$
\gamma(\theta)=\tr(\Sigma_1(\theta)H(\theta)) - \nabla \cL(\theta)^\top H(\theta)\nabla \cL(\theta).
$
It is evident that the second term is clear and what remains is to estimate $\gamma_1(\theta)=\tr(\Sigma_1(\theta)G(\theta))$ by assuming the closeness between $G(\theta)$ and $H(\theta)$.  If the approximation \eqref{eqn: decoupl-x1} holds, we have $\gamma_1(\theta)\approx 2\cL(\theta)\|G(\theta)\|_F^2=:\bar{\gamma_1}(\theta)$. We thus can define the following quantity to measure the influence of noise geometry on loss dynamics:
\begin{definition}[Loss alignment]
$\mu(\theta) = \frac{\gamma_1(\theta)}{\bar{\gamma}_1(\theta)}$.
\end{definition}
At global minima $\theta^*$,  $\gamma_1(\theta)=\bar{\gamma}_1(\theta^*)=0$ and we define $\gamma(\theta^*)=\frac{0}{0}=1$ for convention. Under this definition, $\mu(\cdot)$ may not be continuous but we will show it holds that $\mu(\theta)\sim 1$.

\subsection{(Over-parameterized) linear models}
% \vspace*{-.5em}
% Then we extend the result from linear models to general OLMs with parameter $\theta\in\bbR^p$.

% \begin{theorem}[Non-asymptotic version]\label{thm: lower bound of empirical t1: OLM}
% Let $  x_1,\cdots,  x_n\overset{i.i.d.}{\sim}\mathcal{P}$ satisfying Assumption~\ref{ass: data weak}. For any $\epsilon,\delta>0$, if $n\gtrsim\sbracket{{\sbracket{d+\log(1/\delta)}\cond^2(  S)}/{\epsilon^2}}$ and
% $d\gtrsim\sbracket{({\sigma_{\max}^2\sqrt{\log(n/\delta)}}/{\epsilon})^{1/\alpha}}$, then with probability at least $1-\delta$, we have $\mu(\theta)\geq\frac{(1-\epsilon)^2}{(1+\epsilon)(\cond(  S)+\epsilon)\cond^2(\nabla F(\theta)\nabla F(\theta)^\top)}$ for any $\theta\in\bbR^p$.
% \end{theorem}

% It was derived in \cite{wu2022does}  that for OLMs, when $n=\infty$ (i.e., online SGD), it holds that
% \begin{equation}\label{equ: closed form Sigma, online Gaussian}
% \Sigma_1(\btheta)=2\cL(\btheta)G(\btheta)+2\nabla\cL(\btheta)\nabla\cL(\btheta)^\top.
% \end{equation} 

% This analytical expression \eqref{equ: closed form Sigma, online Gaussian} guarantees $\mu(\theta)\approx 1$ in an infinite data scenario. The following theorem extends it to finite-sample cases and the proof can be found in Appendix~\ref{appendix: proof: weak align}. 

We first consider OLMs, for which we have:

\begin{theorem}[OLM]\label{thm: weak: OLM}
Consider OLMs and suppose Assumption \ref{assumption: input} holds. For any $\epsilon,\delta\in(0,1)$, 
if $n\gtrsim\max\{(d^2\log^2\sbracket{1/{\epsilon}}+\log^2(1/\delta))/\epsilon, (d\log\sbracket{1/{\epsilon}}+\log(1/\delta))/\epsilon^2\},$
then \wp at least $1-\delta$, it holds for any $\theta\in\bbR^p$ that $\frac{1-\epsilon}{(1+\epsilon)^2}\leq\mu(\theta)\leq\frac{2+\epsilon}{(1-\epsilon)^2}$.
\end{theorem}
% \wmz{for Thm 3.2, we have two versions: 
% (I) If $n\gtrsim d^2$, then $\mu(\theta)\sim1$ holds uniformly.
% (II) (new) If $n\gtrsim d$ and $d\gtrsim\log n$, then $1\lesssim \mu(\theta)\lesssim d$ holds uniformly.}

This theorem shows that the alignment strength is well-controlled and notably, the condition is {\em independent} from the re-parameterization map and the number of parameters. Hence, it  can be effectively applied to linear networks regardless of the width and depth.

The following theorem shows that the sample size can be further relaxed for linear models.
\begin{theorem}[Linear model]\label{thm: weak: LM}
Consider linear models and suppose Assumption \ref{assumption: input} holds. For any $\epsilon,\delta\in(0,1)$, 
if $n/\log(n/\delta)\gtrsim 1/\epsilon^2$ and $d_{\mathrm{eff}}\gtrsim\log(n/\delta)/\epsilon^2$, then \wp at least $1-\delta$, it holds for any $\theta\in\bbR^d$ that $\frac{(1-\epsilon)^2}{(1+\epsilon)^2}\leq\mu(\theta)\leq\frac{(1+\epsilon)^2}{(1-\epsilon)^2}$.
\end{theorem}
This theorem is established by leveraging the high dimensionality of inputs, as stated by the condition $d_{\mathrm{eff}}\gtrsim \log n$. Notably, this condition includes the important regimes like $n\sim \log (d_{\mathrm{eff}})$ (for sparse recovery) and $n\sim d_{\mathrm{eff}}$ (the proportional scaling). 

We remark that  the  two theorems are complementary for linear models. 
Specifically, consider the isotropic case where $S=I_d$. 
The two theorems together can cover $n\gtrsim 1$
(Theorem \ref{thm: weak: LM} holds for $n\lesssim e^d$, and Theorem \ref{thm: weak: OLM} holds for $n\gtrsim d^2$).

\subsection{Two-layer neural networks}
Consider two-layer neural networks given by $f(x;\theta)=\sum_{k=1}^m a_k\phi(b_k^\top x)$ with $a_k\in\{\pm1\}$ to be fixed. We use $\theta=(b_1^\top,\cdots,b_m^\top)^\top\in\bbR^{md}$ to denote the concatenation of all trainable parameters. Here, $\phi:\RR\mapsto\RR$ is an activation function with a non-degenerate derivative as defined below.

\begin{assumption}\label{ass: two-layer activation}
There exist constants $\beta>\alpha>0$ such that $\alpha\leq\phi'(z)\leq\beta$ holds for any $z\in\RR$.
\end{assumption}

% \begin{example}
% (i) 
A typical nonlinear activation function that satisfies Assumption \ref{ass: two-layer activation} is $\alpha$-Leacky ReLU: $\!\phi(z)\!=\!\max\{\alpha z,\!z\}$ with $\alpha\!>\!0$. 
% (ii) Moreover, the assumption also holds for Sigmoid with the truncation trick (to prevent gradient vanishing of Sigmoid): $\phi(z)=1/(1+\exp(-\sgn(z)\min\{|z|,M\}))$, where $M>0$ is the truncation constant.
    
% \end{example} 

\begin{theorem}[Two-layer network]\label{thm: weak: 2NN}
Consider the two-layer network $f(\cdot;\theta)$. Suppose Assumption~\ref{ass: two-layer activation} and  Assumption \ref{assumption: input} hold. For any $\epsilon,\delta\in(0,1)$, if $n/\log(n/\delta)\gtrsim 1/\epsilon^2$ and $d_{\mathrm{eff}}\gtrsim\log(n/\delta)/\epsilon^2$, then \wp at least $1-\delta$, it holds that for any $\theta\in\RR^{md}$ that 
$\frac{\alpha^2(1-\epsilon)^2}{\beta^2(1+\epsilon)^2}\leq \mu(\theta)\leq \frac{\beta^2(1+\epsilon)^2}{\alpha^2(1-\epsilon)^2}$.
\end{theorem}

% \begin{theorem}\label{thm: lower bound of empirical t1: two-layer net}
% Consider two-layer networks with the activation function satisfies Assumption \ref{ass: two-layer activation},
% Let $  x_1,\cdots,  x_n\overset{i.i.d.}{\sim}\mathcal{P}$ satisfying Assumption~\ref{ass: data weak}. 

% (i). For any $\epsilon,\delta>0$ and $\gamma\in(0,\alpha)$, if $d\gtrsim\sbracket{\sbracket{{\sigma_{\max}^2\log(n/\delta)}/{\epsilon}}^{1/\gamma},({\sigma_{\max}^2\sqrt{\log(n/\delta)}}/{\epsilon})^{1/\alpha}}$, then with probability at least $1-\delta$, we have
% $\inf_{\theta\in\bbR^{md}} t(\theta)\geq\frac{\alpha^2}{\beta^2}\frac{1-\epsilon}{1+\epsilon\sbracket{1+\frac{n-1}{d^{\alpha-\gamma}}}}$.

% (ii). For any $\epsilon,\delta>0$, if $n\gtrsim\sbracket{{\sbracket{d+\log(1/\delta)}\cond^2(  S)}/{\epsilon^2}}$ and $d\gtrsim\sbracket{({\sigma_{\max}^2\sqrt{\log(n/\delta)}}/{\epsilon})^{1/\alpha}}$, then with probability at least $1-\delta$, we have 
%     $\inf_{\theta\in\bbR^{md}} t(\theta)\geq\frac{\alpha^2}{\beta^2}\frac{(1-\epsilon)^2}{(1+\epsilon)(\cond(  S)+\epsilon)}$.
% \end{theorem}

This theorem establishes a uniform control for the  alignment strength $\mu(\theta)$. Particularly, the number of samples required is {\em independent} of the network width $m$. The proof follows a similar approach to that of Theorem~\ref{thm: weak: LM} and can be found in Appendix~\ref{appendix: proof: weak align}. We remark that the conditions such as the non-degeneracy in activation function's derivatives are obligatory  for establishing alignment across the {\em entire loss landscape}. 
In practice, such stringent conditions may not be necessary, as the focus is on  regions  navigated by SGD. 
% Figure \ref{fig: 1-b} corroborates that alignment is indeed observed for  standard two-layer ReLU networks trained by SGD.
% By incoporating Theorem \ref{thm: weak: 2NN} into the stability analysis developed in \cite{wu2022does}, one can deduce that SGD can only select flat minima for the  over-parameterized two-layer networks.

% \vspace*{-.5em}
\subsection{Numerical validations}
% \vspace*{-.5em}

\begin{figure*}[!hb]
    % \vspace{-.1cm}
    \centering
    \subfloat[$4$-layer Linear networks]{\label{fig: 1-a}
   \includegraphics[width=0.245\textwidth]{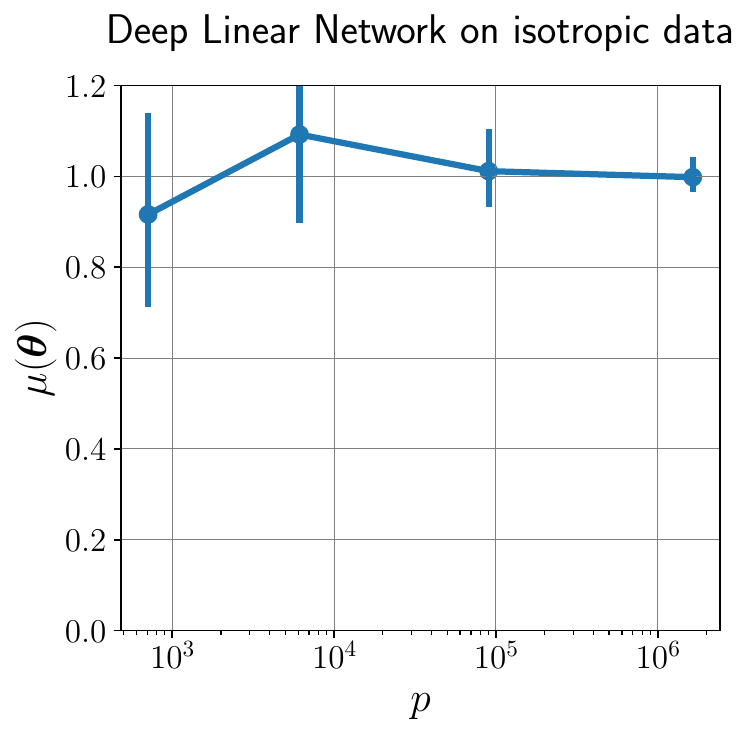}
    \hspace{-.2cm}
    \includegraphics[width=0.25\textwidth]{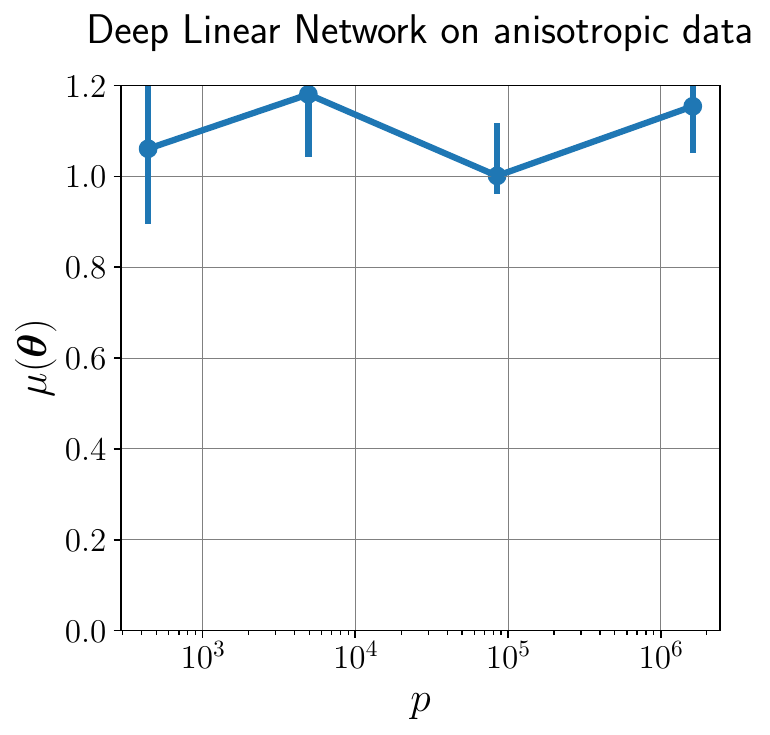}
    }
    \subfloat[Two-layer ReLU networks]{\label{fig: 1-b}
    \includegraphics[width=0.24\textwidth]{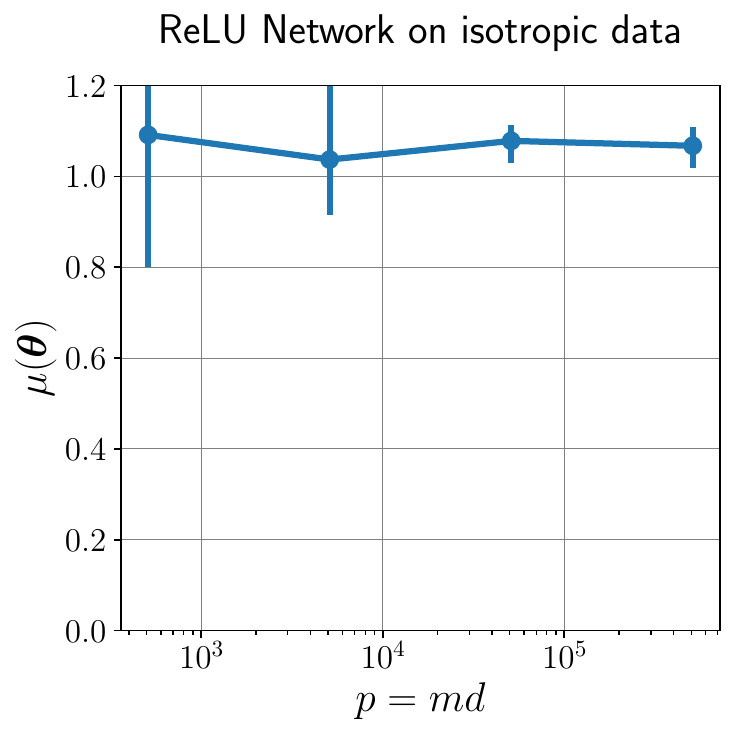}
    \hspace{-.2cm}
    \includegraphics[width=0.235\textwidth]{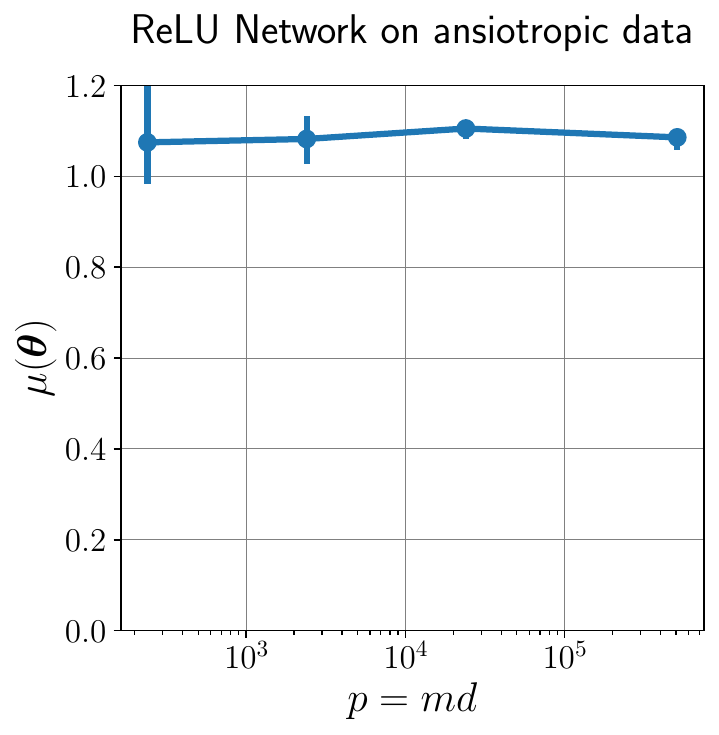}
    }
    \vspace*{-.5em}
    \caption{\small The alignment strength $\mu(\theta)$ is close to $1$ for various models across different model sizes.
    For all experiments, we set $n=5\log(d_{\mathrm{eff}}), d_{\mathrm{eff}}=50$. The input data are drawn from $\cN(0,S)$. For isotropic data,  $S=I_{50}$; for anisotropic data, $S=\diag(\lambda_1,\dots,\lambda_D)$  with $\lambda_k=1/\sqrt{k}$ for $k\in [D]$ where $D$ is chosen such that $d_{\mathrm{eff}}=50$. 
    The error bar corresponds to the standard deviation over $20$ independent runs. 
    The targets are generated by a linear model, i.e., $y_i=\langle w^*,x_i\rangle$, where $w^*\sim N(0,I_d)$.
    We compute $\mu(\theta)$ for randomly chosen $\theta$'s. 
    }
    % Additionally, for (1st), we fix $n/\log d=5$ and increase $d$, whereas for (2nd)(3rd), we fix $n=5\log d$ and $d=50$, and increase the width of networks; for (4th), we fix $\srk(S^2)/\log d =5$ and increase $d$. 
    \label{fig: weak alignment}
    \vspace*{-.5em}
\end{figure*}

Here we present small-scale experiments to validate our theoretical results with a $4$-layer linear network and  two-layer ReLU network (both layers are trainable). Both isotropic and anisotropic input distributions are examined and  we set $n=5\log (d_{\mathrm{eff}})$ to focus on the low-sample regime. The results are reported in Figure \ref{fig: weak alignment} 
% \wl{what is the target function?} 
and it is evident that across all examined scenarios,  the alignment strength is consistently well-controlled and independent of the model size.

\section{Directional alignment} \label{section: theory, strong align}

In this section, we delve  into a type of directional alignment: {\em whether noise energy along a given direction  is proportional to the curvature of loss landscape along that direction}.  Specifically, we use the following metric to quantify the alignment strength.
\begin{definition}[Directional alignment]\label{equ: def: strong align, any direction}
Given $ v\in\bbR^p$, the alignment along $ v$ is defined as 
$
    g(\theta; v):=\frac{ v^\top\Sigma_1(\theta) v}{2\cL(\theta)\sbracket{ v^\top G(\theta) v}},
$
where $v^\top G(\theta) v$ denotes the curvature of local landscape along $v$.
\end{definition}
Denoting by $z_t = \theta_t^\top v$ the component  along $v$, we have
$$
	\EE[z_{t+1}^2] = \EE[\cG_v(z_t)] + \eta^2\EE[(v^\top\xi_t)^2],
	$$
	where $\cG_v(z)=(z-\eta v^\top\nabla \cL(\theta))^2$ represents the contribution from GD part. The additional term $\EE[(v^\top\xi_t)^2]$ arises due to the injected noise, satisfying  $\EE[|\xi(\theta)^\top v|^2]=v^\top \Sigma_1(\theta_t)v - (v^\top\nabla \cL(\theta_t))^2=2g(\theta_t,v)\cL(\theta_t)v^\top G(\theta_t)v - (v^\top\nabla \cL(\theta_t))^2$ by Definition \ref{equ: def: strong align, any direction}. Therefore, the alignment defined above is crucial in describing the dynamics of SGD in  subspaces.

	% According to \eqref{eqn: decoupl-subspace}, $\EE[(v^\top\xi_t)^2]$ is well-controlled by $2\cL(\theta_t)v^\top G(\btheta_t)v$ if the directional alignment holds.

% \begin{theorem}[One-sided bound]\label{thm: strong: lower bound}
% Consider OLMs and suppose Assumption \ref{assumption: input} holds. 

% \end{theorem}

% We remark that this lower bound of alignment is often sufficient to establish an necessary condition for the stability of SGD as done in \cite{}

% \wl{delete the one-sided bound for consistency.}
\begin{theorem}[OLM]\label{thm: strong: OLM}
Consider OLMs and suppose Assumption \ref{assumption: input} holds. For any $\epsilon,\delta\in(0,1)$, if $n\gtrsim\max\{(d^2\log^2\sbracket{1/{\epsilon}}+\log^2(1/\delta))/\epsilon, (d\log\sbracket{1/{\epsilon}}+\log(1/\delta))/\epsilon^2\},$ then \wp at least $1-\delta$,  it holds for any $\theta,v\in\RR^p$ that
$\frac{1-\epsilon}{(1+\epsilon)^2}\leq g(\theta; v)\leq\frac{2+\epsilon}{(1-\epsilon)^2}.$
% \item {\em (One-sided bound)} For any $\delta\in(0,1)$, if $n\gtrsim d+\log(1/\delta)$, then \wp at least $1-\delta$, we have $\inf_{\theta, v\in\bbR^p}g(\theta; v)\gtrsim 1$.
% \end{itemize}

\end{theorem}
This theorem establishes  a uniform control for the alignment across all directions and the entire parameter space. Notably, the condition is independent of the degree of overparameterization. 
The subsequent theorem shows that if focusing on some specific directions, the sample size $n$ can be further relaxed.

\begin{theorem}[Linear model]\label{thm: strong: LM} 
% \wl{Does Thm\ref{thm: strong: fixed dir} hold for the entire subspace spanned by $\cV$?}
% \wmz{(a) Since the dependence on $\bv$ is quadratic, this Thm does not seem to hold for the entire subspace $\cV$.}
% \wmz{(b) Assume this Thm holds for the subspace $\cV$, taking $K=d$, then $n>d\log d$ is enough for $\bbR^d$. It seems impossible.}
Consider linear models and suppose Assumption \ref{assumption: input} holds. Let $\cV=\{v_1,\cdots,v_K\}$ be $K$ fixed directions of interest.
For any $\epsilon\in(0,1/2)$, there exist $C_1=C_1(\epsilon)>0$ and $C_2=C_2(\epsilon)>0$ such that if $n\gtrsim C_1 d\log d$, then \wp at least $1-\frac{K C_2}{n^2}-\epsilon^{d}$, it holds for any $\theta\in\RR^p$ and $v\in \cV$ that
\begin{align*}
    \frac{1-\epsilon}{(1+\epsilon)^2}\leq g(\theta; v)\leq\frac{2+\epsilon}{(1-\epsilon)^2}.
\end{align*}
\end{theorem}

It is worth noting that the above theorems establish the directional alignment for the {\em entire parameter space}.  However, in practice, what matter are typical regions navigated by SGD. This is the gap between our theory and the practice. 
% Our experiments in Figure \ref{fig: strong-alignment-OLMs} show that indeed the directional alignment holds very well for SGD solutions and eigen-directions even when $n\ll d$. 
% On the one hand, to formalize this insight into a theorem is challenging as it requires a precise characterization what “SGD solutions” means. On the other hand, our theorems are also more general in the sense that it reveals that the alignment property is a intrinsic property of mini-batch noise and applicable to optimizers beyond SGD.
% \end{remark}

\begin{figure*}[!ht]
\captionsetup[subfloat]{farskip=.3pt,captionskip=0.2pt}
\centering

\subfloat[Linear models under different regimes]{\label{fig: 2-a}
\hspace*{-1em}
\includegraphics[width=0.245\textwidth]{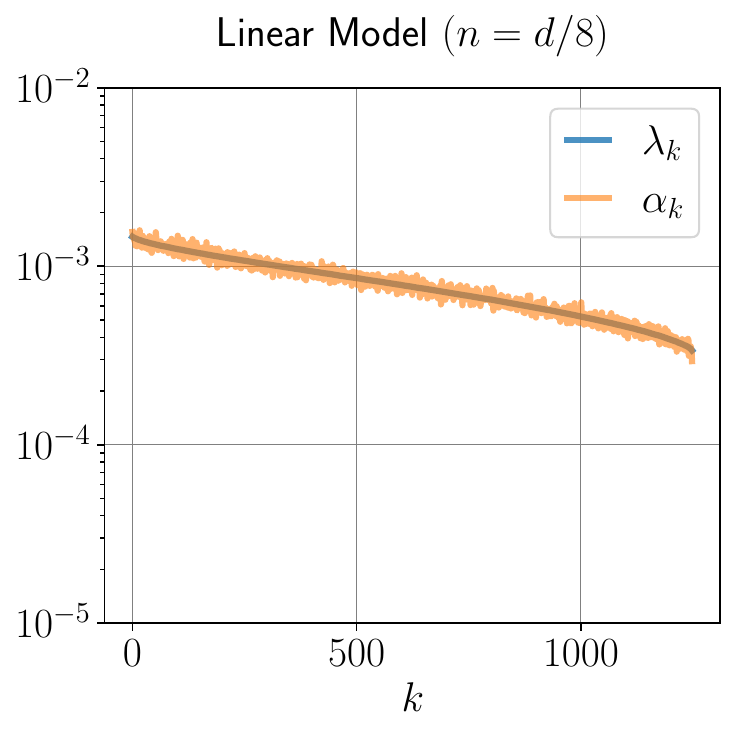}
\includegraphics[width=0.245\textwidth]{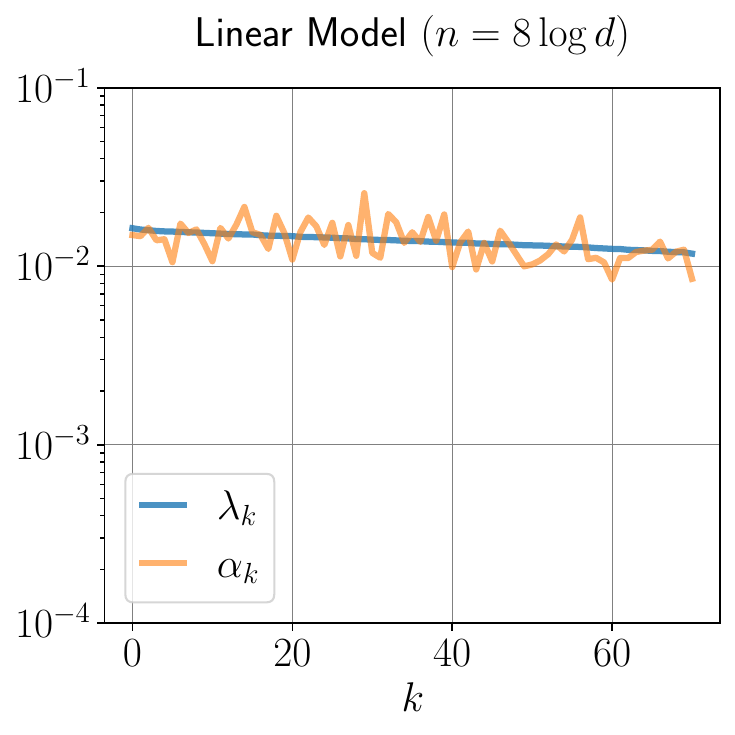}
}
\subfloat[Small neural networks]{\label{fig: 2-b}
   \includegraphics[width=0.245\textwidth]{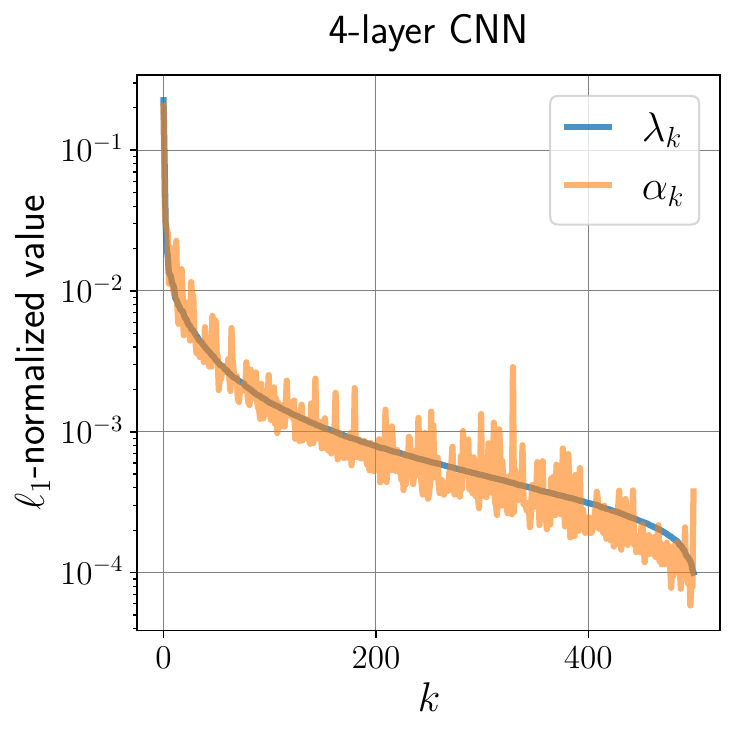}
  \includegraphics[width=0.245\textwidth]{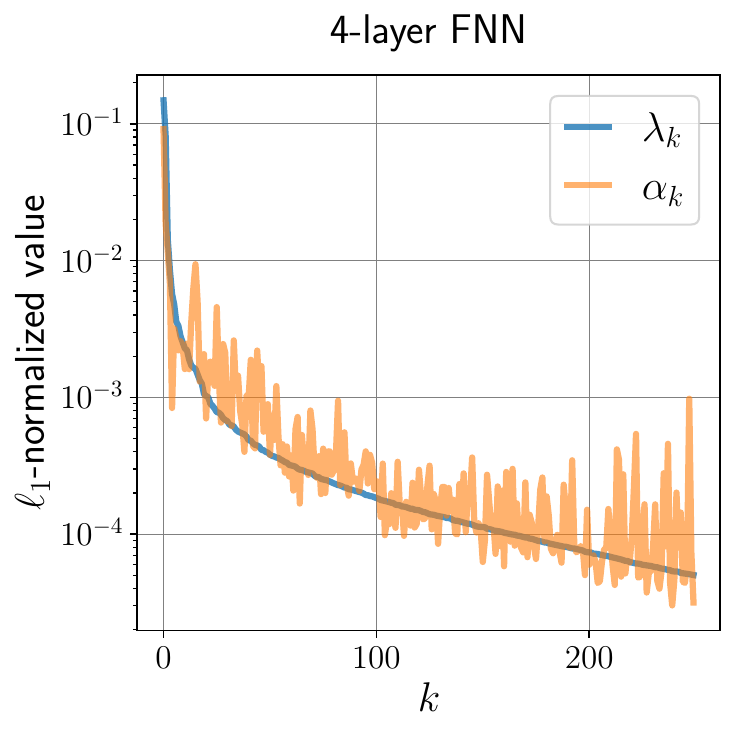}
}
% \hspace*{-1em}

\vspace*{-.2em}
\caption{\small 
How the components of noise energy in \textit{eigen-directions} $\{\alpha_k\}_k$ are proportional to the corresponding curvatures $\{\lambda_k\}_k$. $\alpha_k/\lambda_k$ can reflect the
directional alignment~\eqref{equ: def: strong align, any direction} along the eigen-directions. 
(a) Linear models on Gaussian data in the regimes with limited data, where we fix $d=10^4$ and set $n$ accordingly $(n=d/8,n=8\log d)$.
(b) 4-layer CNN and 4-layer FNN on CIFAR-10 dataset.
For more experimental details, we refer to Appendix~\ref{appendix: exp: setups}.
}
\label{fig: strong-alignment-OLMs}
% \vspace*{-1em}
\end{figure*}

{\bf Numerical validations.} 
In this experiment,   we consider the alignment along the eigen-directions of Hessian matrix. 
Let $G(\theta)=\sum_{k}\lambda_k  u_k  u_k^\top
$ be the eigen-decomposition of $G(\theta)$  where $\{\lambda_k\}_k$ are the eigenvalues in a decreasing order and $\{  u_k\}$ are the corresponding eigen-directions.  Here we omit the dependence of $\theta$ in $\{(\lambda_k,u_k)\}_k$ for brevity. 
Note that $\lambda_k$ is the curvature of local landscape along $  u_k$.  Denote by $\alpha_k = \bbE[|\xi^\top u_k|^2]/(2\cL(\theta))$ the relatively size of noise energy along $u_k$. Then, the directional alignment~\eqref{equ: def: strong align, any direction} satisfies $g(\theta,u_k)=\alpha_k/\lambda_k$.

In Figure \ref{fig: 2-a}, we examine with linear regression under the regimes of limited data, which is beyond our theorems. We still observed that $g(\theta,v)$ is close $1$ for all eigen directions $v$, which is consistent with our theoretical findings. 
% However, we noticed that the eigenvalues of $\Sigma(\theta)/2\cL(\theta)$ decayed much faster than that of $G(\theta)$, indicating that the condition $n\gtrsim d$ stated in Theorem \ref{thm: strong: lower bound} is necessary to ensure uniform alignment across all directions. 
In Figure \ref{fig: 2-b}, we further consider the classification of CIFAR-10 with a small convolutional neural  network (CNN) and fully-connected neural network (FNN). We can see that the observation is consistent with Figure \ref{fig: 2-a}, where the alignments in all eigen-directions are well-controlled, though the fluctuations become more significant.
\section{How SGD escapes from sharp minima}
\label{section: escaping}

In this section, we provide a fine-grained analysis of how SGD escapes from sharp minima by leveraging the  directional alignment of noise geometry. In contrast to existing analyses \citep{zhu2019anisotropic,mori2021logarithmic,wu2022does,xie2020diffusion,kleinberg2018alternative} which only considered the escape rate, we also delve into the escape directions. 

Let $\theta^*$ be a minimum of interest and consider the interpolation regime, i.e., $\cL(\theta^*)=0$. The local escape behavior can be characterized by linearizing the SGD dynamics, which corresponds to the linearized model $f(\cdot;\theta)\approx f(\cdot;\theta^*)+\<\nabla f(\cdot;\theta^*),\theta-\theta^*\>$. We refer to \citep[Section 3.2]{wu2022does} for more details. Thus, without loss of generality, we can simply consider the linearized model in the subsequent analysis.
Let $  {w} = \theta-\theta^*$ and $G(\theta^*)=\frac{1}{n}\sum_{i=1}^n  \nabla f(  x_i;\theta^*) \nabla f(  x_i;\theta^*)^\top$. Then, for the linearized model, we have $\cL(  {w})=\frac 1 2  {w}^\top G(\theta^*)  {w}$ and $\nabla \cL(w)=G(\theta^*)w$. Thus, the linearized SGD updates as follows
\[
  {w}_{t+1}=  {w}_t-\eta\left(G(\theta^*)  {w}_t+  \xi_t\right),
\]
where $  \xi_t$ is the SGD noise. 
We make the following assumption on the noise geometry. 
\begin{assumption}[Eigen-directional alignment]\label{assumption: eigen-alignment}
let $G(\theta^*)=\sum_{i=1}^p\lambda_i  u_i  u_i^\top$ be the eigen decomposition of $G(\theta^*)$.
Assume that there exist $A_1, A_2>0$ such that it holds for any $  {w}\in\RR^p$
\[
A_1 \cL(  {w})\lambda_i \leq \EE[|\xi(  {w})^{\top}  u_i|^2]\leq A_2 \cL(  {w}) \lambda_i.
\]
\end{assumption}
Section \ref{section: theory, strong align} has provided both theoretical and experimental evidence in support of this assumption. For the sake of clarity, we explicitly state it here as an underlying assumption.
% Theorem \ref{thm: strong: fixed dir} has established that for Gaussian inputs, the above assumption holds as long as $n\gtrsim p\log p$. For general linearized models, the inputs are the tangent features $\{\nabla f(x_i;\theta^*)\}_{i=1}^n$, for which the empirical evidence in Figure \ref{fig: 2-b} also confirm this assumption for fully-connected networks and CNNs when trained via SGD.

% however, that the above assumption only requires the alingment along eigen-directions, which is considerably less stringent  compared to the uniform directional alignment specified in Theorem \ref{thm: strong: arb dir}. Consequently, it is plausible that Assumption \ref{assumption: eigen-alignment} enjoys broader applicability. 

{\bf Eigen-decomposition of SGD.}~ By leveraging Assumption \ref{assumption: eigen-alignment}, we can analyze the SGD dynamics in the eigenspace. Let $  {w}_t=\sum_{i=1}^d w_{t,i}  u_i$ with  $w_{t,i}=  u_i^{\top}{w}_t$. Then, $w_{t+1,i}=(1-\eta\lambda_i)w_{t,i}+\eta \xi_t^\top u_i$. Taking the expectation of the square of both sides, we obtain
\begin{align}\label{eqn: 000}
     \bbE\big[w_{t+1,i}^2\big]=(1-\eta\lambda_i)^2\bbE\big[w_{t,i}^2\big]+\eta^2\EE[|  u_i^{\top} \xi_t|^2],
\end{align}
where the noise term: $\EE[|u_i^{\top}  \xi_t|^2]\sim \lambda_i \cL(w_t)$ according to Assumption~\ref{assumption: eigen-alignment}.

Let $X_t=\sum_{i=1}^k\lambda_i\EE[w_{t,i}^2], Y_t = \sum_{i=k+1}^d \lambda_i \EE[w_{t,i}^2]$,   denoting the components of loss energy along sharp and flat directions, respectively. Let  $D_{t,k}=Y_t/X_t$, which
% \begin{equation}\label{eqn: 111}
% \vspace*{-.2em}
% D_k_t=\sum_{i=k+1}^d \lambda_i \EE[w_i^2_t]/\sum_{i=1}^k\lambda_i \EE[w_i^2_t].
% \vspace*{-.5em}
% \end{equation}
% \begin{remark}\label{rmk: escape}
% $D_k_t$
 measures the  concentration  of  loss energy along flat directions. Analogously,
let $P_{t,k}=\sum_{i=k+1}^d\EE[w_{t,i}^2]/\sum_{i=1}^k \EE[w_{t,i}^2]$, which measures the concentration of variance along flat directions. It is easy to show that $P_{t,k}\geq D_{t,k}\lambda_k/\lambda_{k+1}$. Therefore, when $\lambda_k/\lambda_{k+1}$ is lower bounded, a concentration of loss energy along flat directions can lead to a similar concentration in terms of variance.
% \end{remark}

% Plugging the fact that $2\cL(  {w}_t)=X_t+Y_t$ into \eqref{eqn: 000} and by a simple calculation, we have
% \begin{equation}\label{eqn: component}
% \begin{aligned}
% X_{t+1}&\leq \alpha_k X_t + A_2 \eta^2 (\sum_{i=1}^k\lambda_i^2)(X_t+Y_t)\\ 
% Y_{t+1}&\geq \beta_k Y_t + A_1 \eta^2\big(\sum_{i=k+1}^d\lambda_i^2\big) (X_t+Y_t),
% \end{aligned}
% \end{equation}
% where $\alpha_k:= \max_{i=1,\dots,k} |1-\eta\lambda_i|^2$ and $\beta_k:= \min_{i=k+1,\dots,d} |1-\eta\lambda_i|^2$. The terms $\alpha_k X_t$ and $\beta_k Y_t$ capture the impact of the gradient, while the remaining terms originate from the noise.

% Intuitively, if the gradient terms are negligible  compared to the noise terms during the escape process, we trivially have $D_k_t=Y_t/X_t\gtrsim \sum_{i=k+1}^d\lambda_i^2/\sum_{i=1}^k\lambda_i^2$. The following theorem formalize this intuition and 

% Through a detailed analysis, we demonstrate that this indeed holds under a mild condition on the learning rate. We present the formal statement of this theorem below, and 
% its 

% By Assumption \ref{assumption: eigen-alignment}, we have
% \begin{align*}
%      \bbE\big[w_i^2(t+1)\big]=(1-\eta\lambda_i)^2 \bbE\big[w_i^2_t\big]+C_{t,i}\eta^2\lambda_i^2\Big(\sum_{j=1}^d\lambda_jw_j^2_t\Big),
% \end{align*}
% where $1-\epsilon\leq C_{t,i}\leq 2+\epsilon$ holds some absolute $\epsilon>0$ and any $t\in \NN$ and $i\in [n]$.
% This implies that the update of $\bbE[  {w}^2_t]:=(\bbE[w_1^2_t],\cdots,\bbE[w_d^2_t])^\top$ is approximately linear.

\begin{theorem}[Escape of SGD]\label{thm: escape: SGD}
Suppose Assumption \ref{assumption: eigen-alignment} holds and let $\eta = \frac{\beta}{\|G(\theta^*)\|_{\rm F}}$. Then, there exists absolute constants $c_1,c_2>0$ such that if $\beta\geq c_1$, then SGD will escape from that minima and  for any $k\in [d]$, it holds that  when $t\geq \max\Big\{1, \frac{\log\big(c_2/\eta(\sum_{i=1}^k\lambda_i^2)^{1/2}\big)}{\log\beta}\Big\}$:
$
    D_{t,k}\gtrsim \frac{\sum_{i=k+1}^d\lambda_i^2}{\sum_{i=1}^k\lambda_i^2}.
$
\end{theorem}
% \wmz{The rate is usually provided to characterize the convergence to a fixed point. However, we need provide a lower bound in terms of $\gtrsim$. It may not be natural to provide rate, because different constant in $\gtrsim$ can result in different rates...}
% \wmz{however, it is easy to provide a rate if necessary. According to our proof, there are two cases. If eta is large, the result only need one step; If eta is small, it seems exponentially fast.}
% \wmz{By the way, it seems that we can provide a two-sided bound.}

The proof can be found in Appendix~\ref{appendix: proof: escaping}.
This theorem reveals that during  SGD's escape process, the loss rapidly accumulates a significant component along flat directions of the loss landscape. The precise loss ratio between the flat and sharp directions is governed by the spectrum of Hessian matrix. In particular,  $D_{t,1}\gtrsim\srk(G^2)-1$, indicating that in high dimension, i.e., $\srk(G^2)\gg 1$, the loss energy along the sharpest directions becomes negligible during the SGD's escape process. This stands in stark contrast to GD, which always escapes along the sharpest direction:

\begin{proposition}[Escape of GD]\label{thm: escape GD}
Consider GD with learning rate $\eta=\beta/\lambda_1$. If $\beta>2$, then
$D_{t,1}\leq\sum\limits_{i=2}^d\frac{\lambda_i(1-\eta\lambda_i)^{2t}w_{0,i}^2}{\lambda_1(1-\eta\lambda_1)^{2t}w_{0,1}^2}$.
\end{proposition}

\begin{wrapfigure}{r}{0.49\textwidth}
    \vspace{-2.4em}
    \centering
  \includegraphics[width=0.45\textwidth]{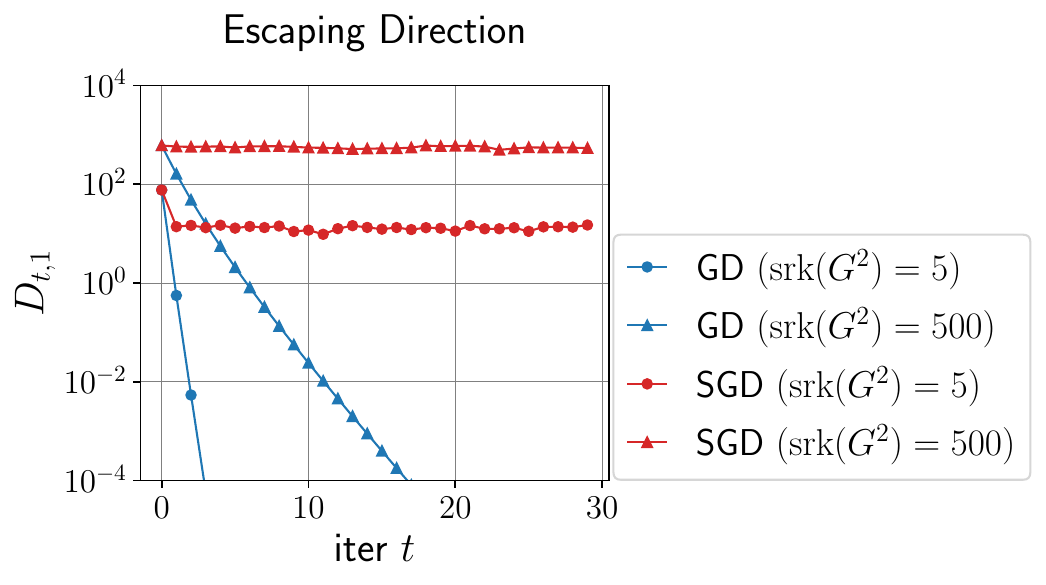}
    \vspace{-.2cm}
    \caption{\small Comparison of escape directions between SGD and GD. The problem is linear regression and both SGD and GD are initialized near the global minimum by $  {w}_0\sim\cN( {w}^*,e^{-10}I_d/d)$. To ensure escape, we choose $\eta=1.2/\norm{G}_{\rF}$ and $\eta=4/(\lambda_1+\lambda_2)$ for SGD and GD, respectively.  Please refer to Appendix~\ref{appendix: exp: setups} for more experimental details.
    }
    \vspace{-.3cm}
    \label{fig: LM-exp escaping}
\end{wrapfigure}
In particular, if $w_{0,1}\neq 0$ and $\lambda_1>\lambda_2$, then the above proposition implies that $D_{t,1}$ decreases to $0$ exponentially fast for GD.

Figure \ref{fig: LM-exp escaping} presents numerical comparisons of the escaping directions between SGD and GD. It is evident that $D_{t,1}$ exponentially decreases to zero for GD, indicating that GD escapes along the sharpest direction. In contrast, for SGD,  $D_{t,1}$ remains significantly large, indicating that SGD retains a substantial component along the flat directions during the escape process. Furthermore, the value of $D_{t,1}$ positively correlates with $\srk(G^2)$, as predicted by our Theorem \ref{thm: escape: SGD}. These observations provide empirical confirmation of our theoretical findings.

% \begin{wrapfigure}{r}{6.0cm}
%     \vspace{-.2cm}
%     \centering
%   \includegraphics[width=5.0cm]{figures/escape/escape_vs_LM.pdf}
%   % \includegraphics[width=3.3cm]{figures/escape/LM_eps01.pdf}
%     \vspace{-.2cm}
%     \caption{\small Escaping Directions of SGD and GD for high-dimensional linear regression with different choices of $P$.}
%     \vspace{-.2cm}
%     \label{fig: LM-exp escaping}
% \end{wrapfigure}

% Moreover, we conduct experiments to validate our theoretical results. 
% Figure~\ref{fig: LM-exp escaping} displays the directions of both SGD and GD escaping from minima $  {w}^*$ for high-dimensional linear regression. One can see that
% (i) for GD, the escape direction $D_1_t$ always exponentially approaches $0$, as shown in Proposition~\ref{thm: escape GD}. This means that GD only escapes along the sharpest direction; 
% (ii) however, for SGD, there exists a significant lower bound on its escape direction $D_1_t$, which means that the escape trajectory of SGD has significant component on flat directions. 

% \paragraph*{Remark on nonlinear models.}

% \wmz{
% With this subtitle, I feel that the study of CLR is crucial in Sec 5 (such as Sec 1.1). But it seems to be just a Case Study.
% }

{\bf Explaining the implicit bias of cyclical learning rate.}
Gaining insights into the escape direction  can be valuable for understanding how SGD explores the non-convex landscape.
% A more detailed discussion on this topic is available in Section \ref{sec: discussion}. 
Specifically, we provide a concrete example to illustrate the role of escape direction in enhancing the implicit bias of SGD through Cyclical Learning Rate (CLR) \citep{smith2017cyclical,loshchilovsgdr}. 
As shown in Figure 2 of \citet{huangsnapshot}, utilizing CLR enables SGD to cyclically  escapes from (when increasing LR) and slides into (when decreasing LR) sharp regions, ultimately progressing towards flatter minima. We hypothesize that  escape along flat directions plays a pivotal role in guiding SGD towards flatter region in this process.

% \vspace*{-.3cm}

\begin{wrapfigure}{r}{0.4\textwidth}
    \vspace{-1em}
    \centering
    \includegraphics[width=4cm]{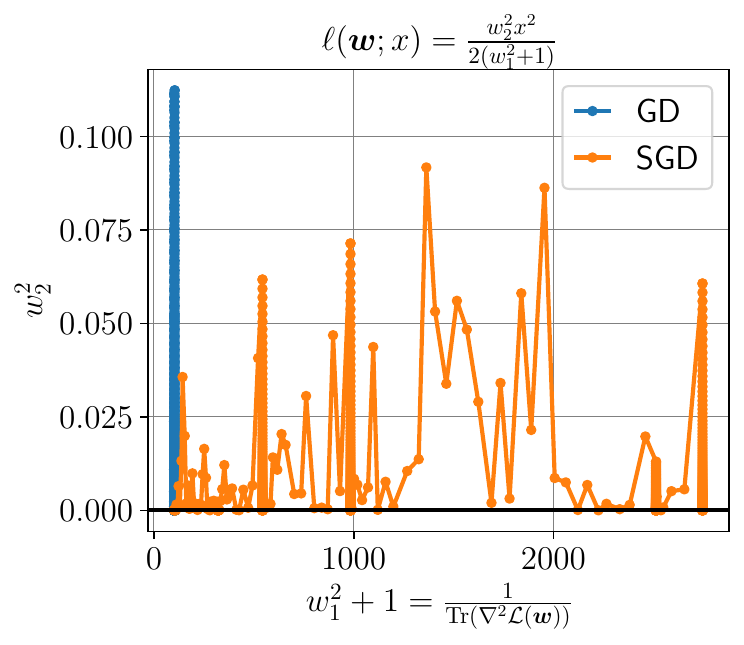}
    \vspace{-0.1cm}
    \caption{\small Visualization of the trajectories of SGD+CLR v.s.~GD+CLR for our toy model. Both cases use the same CLR schedule. We can observe that SGD+CLR moves significantly towards flatter region, while  GD+CLR only oscillates along the sharpest direction. We have extensively tuned the learning rates for GD+CLR but do not observe significant movement towards flatter region in any case.
    }
    \label{fig: cyclical-lr}
\end{wrapfigure}

Following \citet{ma2022beyond}, we 
 consider a toy OLM $f(x; w)=(w_2/\sqrt{w_1^2+1})x$ with $x\sim\mathcal{N}(0,1)$. For simplicity, we consider the online setting, where the  landscape 
 \[
 \mathcal{L}( w)={w_2^2}/{[2(w_1^2+1)]}.
 \]
 The  global minima valley is $S=\{ w : w_2=0\}$ and for ${w}\in S$, $\tr[\nabla^2 \mathcal{L}(w)] = 1/(1+w_1^2)$. Hence, the minimum gets flatter along the valley $S$ when $|w_1|$ grows up. In Figure \ref{fig: cyclical-lr},  we visualize the trajectories for both SGD+CLR and GD+CLR. One can observe that
 \begin{itemize}[leftmargin=2em]
 \item SGD escape from the minima along both the flat direction $e_1$ and sharp direction $e_2$.  The component of along $e_1$ leads to  considerable increase in $w_{1}^2$, facilitating the movement towards flatter region along the minimum valley $S$. 
 \item  On the contrary, GD escapes only along $e_2$, yielding no increase in $w_{1}^2$. Thus, we cannot observe clear movement towards flatter region for GD+CLR. 
 \end{itemize}
 Thus, in this toy model, the fact that SGD escapes along flat directions is crucial in amplifying  the implicit bias towards flat minima. 

 Nonetheless, understanding how the above mechanism manifests in practice remains an open question that warrants further investigation. We defer this topic for future work, as the primary focus of this paper is to understand the noise geometry rather than exhaustively explore its  applications.

\section{Larger-scale experiments for deep neural networks}
\label{section: large-scale exps}
% \vspace*{-.5em}

We have already provided  small-scale experiments to confirm our theoretical findings. 
We now turn to justify the practical relevance of theoretical findings by examining the classification of CIFAR-10 dataset~\citep{krizhevsky2009learning} with practical VGG nets~\citep{vgg} and ResNets~\citep{he2016deep}.  Note that larger-scale experiments on loss alignment have been previously presented in~\citet{wu2022does}. Thus, our focus here is on investigating the directional alignment and escape direction of SGD. We refer to Appendix~\ref{appendix: exp: setups} for experimental details.

{\bf The directional alignment along eigen-directions.}
Figure~\ref{fig: strong-alignment-DNNs} presents the directional alignments of SGD noise for ResNet-38 and VGG-13. The alignment is examined along the eigen-directions of the local landscape. The two quantities: ${\lambda_k}$ and ${\alpha_k}$ under $\ell_1$ normalization (i.e., ${\lambda_k/\norm{  \lambda}_1}$ and ${\alpha_k/\norm{ \alpha}_1}$) are plotted. Here, $\lambda_k$ and $\alpha_k$ represent the curvature and the component of noise energy along the $k$-th eigen-direction, respectively. 
% $\mu_k$ corresponds to the $k$-th eigenvalue of the noise covariance matrix, which is included for comparison. 
One can see that  the alignment between $\alpha_k$ and $\lambda_k$ still exists for ResNet-38 and VGG-13, but the ratio between them becomes significantly larger.  As a comparison, we refer to Figure \ref{fig: 2-b}, where the ratio is well-controlled for small-scale networks trained for classifying the same dataset. We hypothesize that this observation is consistent with our theoretical results in Section \ref{section: theory, strong align}: one-sided bounds require much less samples. 

% suggests that the alignment strength may vary with the model size, where larger models exhibit a weaker alignment between noise energy and curvature.

\begin{figure*}[!ht]
 % \vspace{-.2cm}
    \centering
 
    \includegraphics[width=0.28\textwidth]{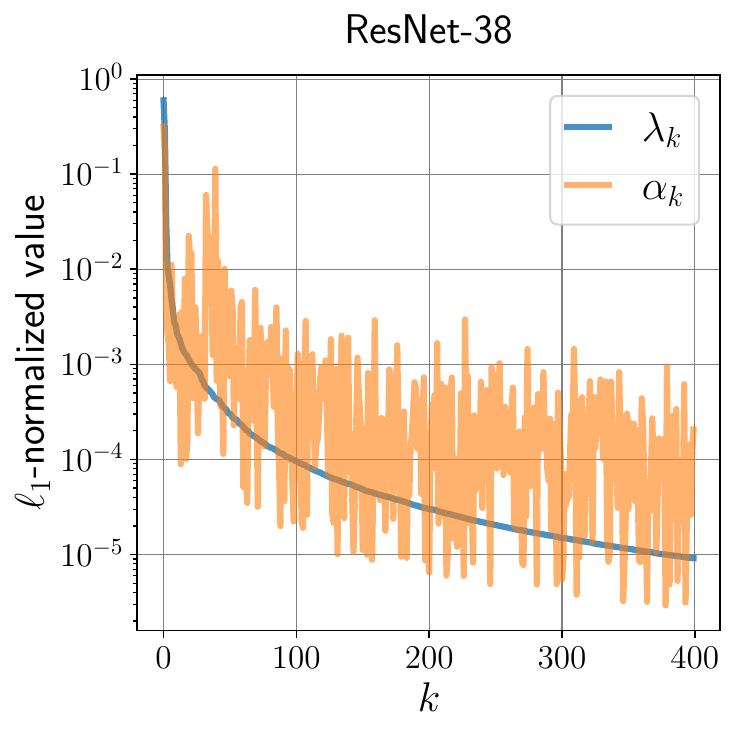}
    \hspace*{2em}
    \includegraphics[width=0.28\textwidth]{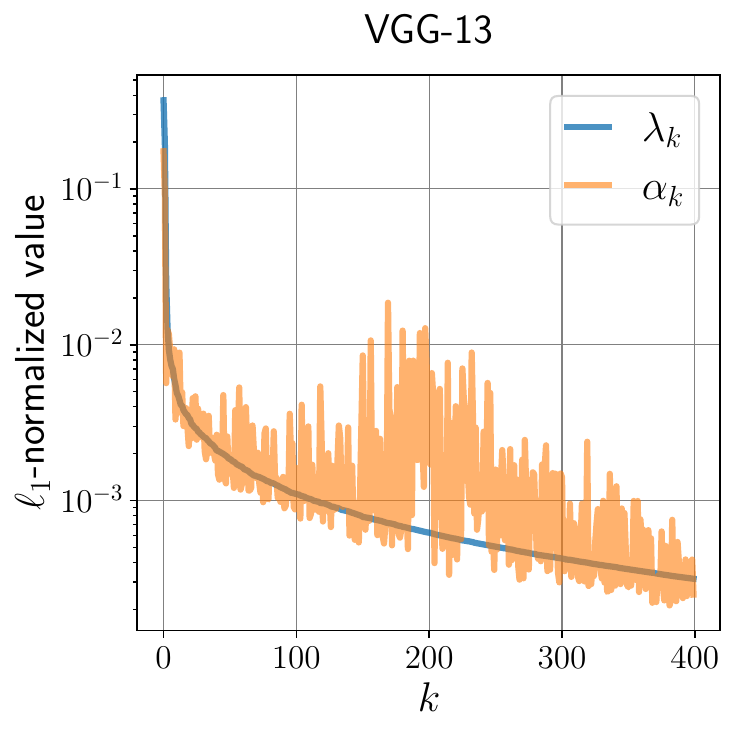}
 \vspace{-.1cm}
    \caption{\small Three distributions ($\{\lambda_k\}_k$ and $\{\alpha_k\}_k$) for larger-scale neural networks, which reflect the directional alignment~\eqref{equ: def: strong align, any direction} along the eigen directions of the local landscape.}
    \label{fig: strong-alignment-DNNs}
     % \vspace{-2em}
\end{figure*}

{\bf The escape direction of SGD.}
For large models, it is computationally prohibitive  to compute the quantity $D_{t,k}$  since it needs to compute the whole spectrum. Thus, we consider to measure the component along different directions without reweighting. 
Let $\theta^*$ be the minimum of interest and $\theta_t$ be SGD/GD solution at step $t$. Define $p_{t,k}=\<\theta_t-\theta^*,  u_1\>$ for $k=1$ and $p_{t,k}=(\sum_{i=1}^k\<\theta_t-\theta^*,  u_i\>^2)^{1/2}$ for $k>1$; $r_{t,k}=(\norm{\theta_t-\theta^*}^2-p_{t,k}^2)^{1/2}$. Notably, $p_{t,k}$ and $r_{t,k}$ represent the component along sharp and flat directions, respectively.

\begin{figure*}[!ht]
 % \vspace{-.2cm}
    \centering
    \includegraphics[width=0.28\textwidth]{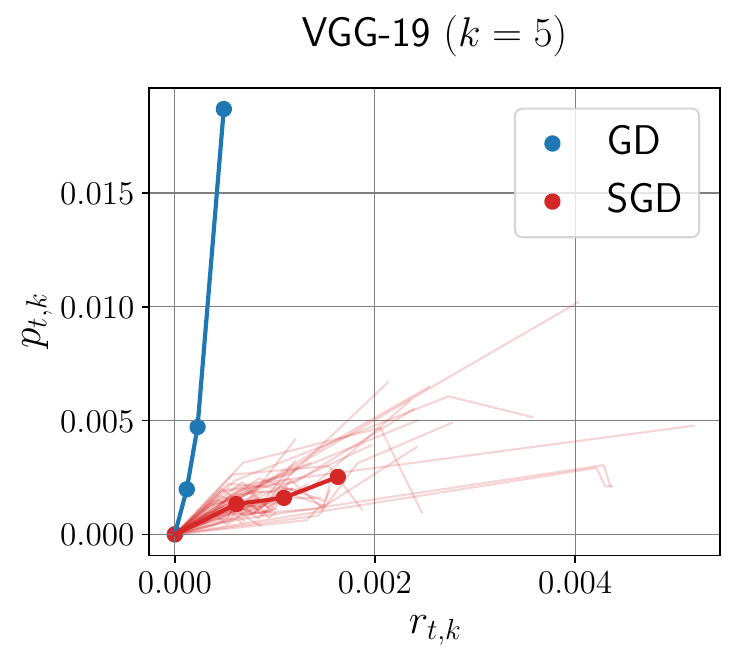}
    \includegraphics[width=0.28\textwidth]{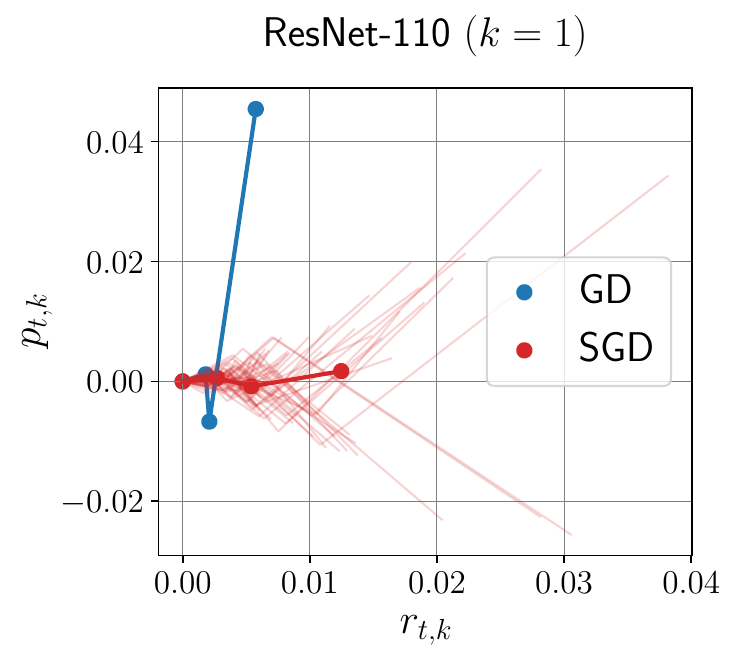}
    \includegraphics[width=0.28\textwidth]{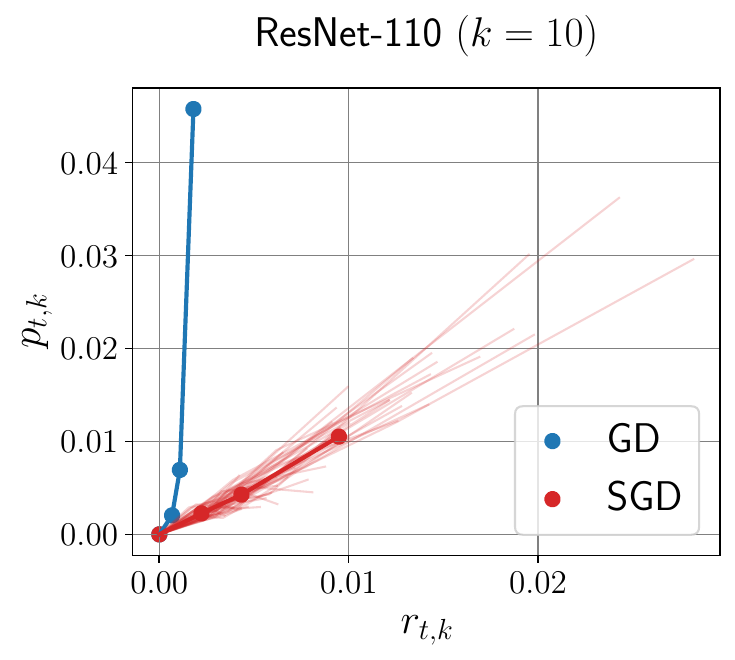}
    \vspace{-.1cm}
    \caption{\small The red curves are $50$ escaping trajectories of SGD and their average; the blue curves corresponding to GD. 
    The sharp minimum $\theta^*$ is found by SGD. Then, we run SGD and GD starting from $\theta^*$ and the learning rates are tuned to ensure escaping.}
    \label{fig: large-exp escaping}
     % \vspace{-1em}
\end{figure*}

In Figure~\ref{fig: large-exp escaping}, we plot $(p_{t,k}, r_{t,k})$  for VGG-19 and ResNet-110, where we examine various $k$ values. The plots clearly demonstrate that the escape direction of SGD exhibits significant components along the flat directions. On the other hand, GD tends to escape along much sharper directions. These empirical findings align well with our theoretical findings in Section \ref{section: escaping}.

\section{Concluding remark} 
\label{sec: discussion}
% \vspace*{-.5em}

In this paper, we present a thorough investigation of the geometry of SGD noise, providing quantitative characterizations of how SGD noise aligns well with the landscape's local geometry. Furthermore, we explore the implications of these findings by analyzing the direction of SGD escaping from shapr minima, as well as its role in enhancing the implicit bias toward flatter minima through cyclical learning rate.

Generally speaking, understanding the noise geometry is crucial for comprehending many aspects of SGD dynamics. Our analysis of SGD's escape direction only serves as a preliminary example of this. Looking ahead, there are numerous potential avenues for further research.   One such direction is to investigate how noise geometry influences the convergence and stability of the loss in SGD, as explored in recent works \cite{thomas2020interplay,wu2022does,wu2023implicit}. Furthermore, as indicated by the analysis in Section  \ref{section: escaping}, SGD noise can substantionally alter the unstable directions. This observation opens up the intriguing possibility of studying the impact of noise geometry on the Edge of Stability (EoS) \citep{cohen2021gradient,wu2018sgd} and  the associated unstable convergence phenomena  of SGD \citep{ahn2022understanding}.

\section*{Acknowledgments}
Mingze Wang is supported in part by the National Key Basic Research Program of China (No. 2015CB856000). Lei Wu is supported by a startup fund from Peking University.
We thank anonymous reviewers for their
valuable suggestions.

% \thankmedskip

% \bibliography{ref}
% \bibliographystyle{plainnat}

%%%%%%%%%%%%%%%%%%%%%%%%%%%%%%%%%%%%%%%%%%%%%%%%%%%%%%%%%%%%%%%%%%%%%%%%%%%%%%%%%%%%%%%%%%%%%%%%%%%%%%%%%%%%%%%%%%%%%%%%%%%%%%%%%%%%%%%%%%%%%%%%%%%%%%%%%%%%%%%%%%%%%%%%%%%%%%%%%%%%

\newpage
\appendix
\newpage
\appendix

\begin{center}
    \noindent\rule{\textwidth}{1.0pt} \vspace{-0.25cm}
    \LARGE \textbf{Appendix} % \\ ~\\[-0.5cm]
    \noindent\rule{\textwidth}{1.0pt}
\end{center}

\startcontents[sections]
\printcontents[sections]{l}{1}{\setcounter{tocdepth}{2}}
% \tableofcontents

\section{Experimental Setups}\label{appendix: exp: setups}

In this section, we provide the experiment details for directional alignment experiments (in Figure~\ref{fig: strong-alignment-OLMs} and Figure \ref{fig: strong-alignment-DNNs}) and escaping experiments (in Figure~\ref{fig: LM-exp escaping} and Figure \ref{fig: large-exp escaping}).

{\bf Small-scale experiments} (Figure~\ref{fig: strong-alignment-OLMs} and~\ref{fig: LM-exp escaping}).
\begin{itemize}
    \item In Figure~\ref{fig: strong-alignment-OLMs} (a), we conduct directional alignment experiments on linear regression. 
    The inputs $\{\bx_i\}_{i=1}^n$ are drawn from $\cN(0,I_d)$. The targets are generated by a linear model, i.e., $y_i={w^*}^\top x_i$, where $w^*\sim N(0,I_d)$.
    We fix $d=10^4$ and change $n$ accordingly $(n=d/8,n=8\log d)$. Regarding the parameter $ \theta$, it is drawn from $\cN( 0,I_{p})$.
    \item In Figure~\ref{fig: strong-alignment-OLMs} (b), we conduct directional alignment experiments on a 4-layer CNN $(p=43,072)$ and a 4-layer FNN $(p=219,200)$. Specifically, the architecture of 4-layer CNN is $\texttt{Conv}(3,6,5)\to\texttt{ReLU}\to\texttt{MPool}(2,2)\to\texttt{Conv}(6,16,5)\to \texttt{ReLU}\to\texttt{MPool}(2,2)\to \texttt{Linear}(400, 100)\to \texttt{ReLU}\to\texttt{Linear}(100, 2)$.
    and the 4-layer FNN is a ReLU-activated fully-connected network with the architecture: $784\to 256\to 64 \to 32\to 2$.  We use the CIFAR-10 dataset with label=$0,1$. Regarding the parameter $\theta$, it is drawn from $\cN( 0,I_{p})$.

    \item In Figure~\ref{fig: LM-exp escaping}, we conduct escaping experiments on linear regression with $ w^*= 0$.
    Both SGD and GD are initialized near the global minimum by $ w_0\sim\cN( 0,e^{-10}I_d/d)$. To ensure escaping, we choose $\eta=1.2/\norm{G}_{\rF}$ and $\eta=4/(\lambda_1+\lambda_2)$ for SGD and GD, respectively. We fix $n=10^5$ and $d=10^3$, and the inputs $\{\bx_i\}_{i=1}^n$ are drawn from $\cN( 0,\diag(\lambda)/d)$, where $\lambda\in\bbR^d$ and $\lambda_1\geq\lambda_2=\cdots=\lambda_d\geq0$. Moreover, we set $\lambda_1=1$ change $\lambda_2$ accordingly to obtain different $\srk(G^2)$.
\end{itemize}

{\bf Larger-scale experiments} (Figure~\ref{fig: strong-alignment-DNNs} and~\ref{fig: large-exp escaping}).

\begin{itemize}
    \item Dataset. For the experiments in Figure~\ref{fig: strong-alignment-DNNs} and~\ref{fig: large-exp escaping}, we use the CIFAR-10 dataset with label=$0,1$ and the full CIFAR-10 dataset to train our models, respectively.
    \item Models. We conduct experiments on large-scale models: ResNet-38 $(p=558,222)$, VGG-13 $(p=605,458)$, ResNet-110 $(p=1,720,138)$, and VGG-19 $(p=20,091,338)$. 
    Specifically, we use standard ResNets~\citep{he2016deep} and VGG nets~\citep{vgg} without batch normalization. 
    For ResNets, we follow~\cite{fixup} to use the fixup initialization in order to ensure that the model can be trained without batch normalization.

    \item Training. All explicit regularizations (including weight decay, dropout, data augmentation, batch normalization, learning rate decay) are removed, and a simple constant-LR SGD is used to train our models. 
    Specifically, all these models are trained by SGD with learning rate $\eta=0.1$ and batch size $B=32$ until the training loss becomes smaller than $10^{-4}$.

\end{itemize}

{\bf Efficient computations} of the top-$k$ eigen-decomposition of $G$ and $\Sigma$.
We utilize the functions \texttt{eigsh} and \texttt{LinearOperator} in \texttt{scipy.sparse.linalg} to calculate top-$k$ eigenvalues and eigenvectors of $G$ and $\Sigma$, and the key step is to efficiently calculate $G\bv$ and $\Sigma\bv$ for any given $\bv\in\bbR^p$.
\begin{itemize}
\item For small-scale experiments, they can be calculated directly. 
\item For the large-scale models, we need further approximations since the computation complexity $\cO(np)$ is prohibitive in this case. 
To illustrate our method, we will use $G\bv$ as an example and apply a similar approach to $\Sigma\bv$.
Notice that the formulation $G\bv=\frac{1}{n}\sum_{i=1}^n(\bx_i^\top\bv)\bx_i$ are all in the form of sample average, which allows us to perform Monte-Carlo approximation. Specifically, we randomly choose $b$ samples $\{\bx_{i_j}\}_{j=1}^b$ from $\bx_1,\dots,\bx_{n}$ and use $\frac{1}{b}\sum_{j=1}^b(\bx_{i_j}^\top\bv)\bx_{i_j}$ estimate $G\bv$, with the computation complexity $\cO(bp)$.
For the experiments on CIFAR-10, we test $b$’s with different values and find that $b=2k$ is sufficient to obtain a reliable approximation of the top-$k$ eigenvalues and eigenvectors. Hence, for all large-scale
experiments in this paper, we use $b=2k$ to speed up the computation of the top-$k$ eigenvalues and eigenvectors.

% We take $G\bv$ as an example to illustrate 

\end{itemize}

\section{Proofs of Section \ref{sec: average}}
\label{appendix: proof: weak align}

\subsection{Proof of Theorem~\ref{thm: weak: OLM}}
This result is a direct corollary of Theorem \ref{thm: strong: OLM}, which is proved in Appendix~\ref{appendix: proof: strong align}.

By Theorem~\ref{thm: strong: OLM}, under the same conditions, the following uniform bound holds:
\begin{align*}
\frac{1-\epsilon}{(1+\epsilon)^2}\leq\inf_{\theta,\bv\in\bbR^p}g( \theta;\bv)\leq\sup_{ \theta,\bv\in\bbR^p}g( \theta;\bv)\leq\frac{2+\epsilon}{(1-\epsilon)^2},
\end{align*}
which means that for any $ \theta\in\bbR^{p},\bv\in\bbS^{p-1}$, we have 
\begin{align*}
    \frac{1-\epsilon}{(1+\epsilon)^2}\cdot 2\cL(\theta)\bv^\top G( \theta)\bv\leq\bv^\top\Sigma_1(\theta)\bv\leq\frac{2+\epsilon}{(1-\epsilon)^2}\cdot 2\cL( \theta)\bv^\top G(\theta)\bv.
\end{align*}

Consider the orthogonal decomposition of $ G(\theta)$: $ G( \theta)=\sum_{k=1}^p\lambda_k u_k u_k^\top$. Notice that
\begin{align*}
&\Tr(\Sigma_1(\theta) G( \theta))=\sum_{k=1}^p\lambda_k u_k^\top\Sigma_1(\theta) u_k,
\\
&\norm{G(\theta)}_{\rF}^2=\Tr(G(\theta) G(\theta))=\sum_{k=1}^p\lambda_k u_k^\top G(\theta) u_k.
\end{align*}

Then we obtain
\begin{align*}
    \Tr(\Sigma_1(\theta) G(\theta))\geq&
    \frac{1-\epsilon}{(1+\epsilon)^2}\cdot 2\cL( \theta)\sum_{k=1}^p\lambda_k u^\top_k G( \theta) u_k=\frac{1-\epsilon}{(1+\epsilon)^2}\cdot2\cL( \theta)\norm{ G( \theta)}_{\rF}^2,
    \\
    \Tr(\Sigma_1(\theta) G(\theta))\leq&
    \frac{2+\epsilon}{(1-\epsilon)^2}\cdot 2\cL(\theta)\sum_{k=1}^p\lambda_k u^\top_k G(\theta) u_k=\frac{2+\epsilon}{(1-\epsilon)^2}\cdot2\cL( \theta)\norm{ G( \theta)}_{\rF}^2,
\end{align*}

which means $\frac{1-\epsilon}{(1+\epsilon)^2}\leq\mu( \theta)\leq\frac{2+\epsilon}{(1-\epsilon)^2}$. From the arbitrariness of $ \theta$, we complete the proof.
\qed

\subsection{Proof of Theorem \ref{thm: weak: LM}}
% For clarity, in a slightly different order from the main text, we first prove for the linear model (Example) and then for the OLM (Theorem~\ref{thm: weak: OLM}). 
% This is also convenient for us to compare the difference between the proof for the two-layer neural network (Theorem~\ref{thm: weak: 2NN}) and the proof for the linear model.

For the linear model, i.e., $ \theta= w$ and $\bF( w)= w$ in OLMs, we have the following lower bound: 
\begin{equation}\label{equ: proof: weak: LM}
    \begin{aligned}
    &\mu ( w)=\frac{\Tr\sbracket{\Sigma_1( w)G( w)}}{2\cL( w)\norm{G( w)}_{\rF}^2}
    \\=&\frac{{\rm Tr}\Bigg(\Big(\frac{1}{n}\sum\limits_{j=1}^n\bx_j\bx_j^\top\Big)\Big(\frac{1}{n}\sum\limits_{i=1}^n(F( \theta)^\top \bx_i)^2 (\nabla F( \theta)^\top\bx_i)(\nabla F( \theta)^\top\bx_i) ^\top\Big)\Bigg)}{\Big(\frac{1}{n}\sum\limits_{i=1}^n(F( \theta)^\top \bx_i)^2\Big)\Big(\frac{1}{n^2}\sum\limits_{i=1}^n\sum\limits_{j=1}^n\big(\bx_i^\top\nabla F( \theta)\nabla F( \theta)^\top \bx_j\big)^2\Big)}
    \\=&\frac{\frac{1}{n^2}\sum\limits_{i=1}^n\sum\limits_{j=1}^n( w^\top \bx_i)^2\big(\bx_i^\top\bx_j\big)^2}
    {\Big(\frac{1}{n}\sum\limits_{i=1}^n( w^\top \bx_i)^2\Big)\Big(\frac{1}{n^2}\sum\limits_{i=1}^n\sum\limits_{j=1}^n\big(\bx_i^\top\bx_j\big)^2\Big)}
    \geq\frac{\Big(\frac{1}{n}\sum\limits_{i=1}^n( w^\top \bx_i)^2\Big)\Big(\min\limits_{i\in[n]}\frac{1}{n}\sum\limits_{j=1}^n\big(\bx_i^\top\bx_j\big)^2\Big)}
    {\Big(\frac{1}{n}\sum\limits_{i=1}^n( w^\top \bx_i)^2\Big)\Big(\frac{1}{n^2}\sum\limits_{i=1}^n\sum\limits_{j=1}^n\big(\bx_i^\top\bx_j\big)^2\Big)}
    \\=&\frac{\min\limits_{i\in[n]}\frac{1}{n}\sum\limits_{j=1}^n\big(\bx_i^\top\bx_j\big)^2}
    {\max\limits_{i\in[n]}\frac{1}{n}\sum\limits_{j=1}^n\big(\bx_i^\top\bx_j\big)^2}
    \geq\frac{\min\limits_{i\in[n]}\norm{\bx_i}^4+(n-1)\min\limits_{i\in[n]}\frac{1}{n-1}\sum\limits_{j\ne i}(\bx_i^\top\bx_j)^2}{\max\limits_{i\in[n]}\norm{\bx_i}^4+(n-1)\max\limits_{i\in[n]}\frac{1}{n-1}\sum\limits_{j\ne i}(\bx_i^\top\bx_j)^2}.
\end{aligned}
\end{equation}

In the same way, the upper bound holds:
\begin{align*}
    \mu (w)=\frac{\Tr\sbracket{\Sigma_1( w)G( w)}}{2\cL( w)\norm{G( w)}_{\rF}^2}
    \leq\frac{\max\limits_{i\in[n]}\frac{1}{n}\sum\limits_{j=1}^n\big(\bx_i^\top\bx_j\big)^2}
    {\min\limits_{i\in[n]}\frac{1}{n}\sum\limits_{j=1}^n\big(\bx_i^\top\bx_j\big)^2}
    \leq\frac{\max\limits_{i\in[n]}\norm{\bx_i}^4+(n-1)\max\limits_{i\in[n]}\frac{1}{n-1}\sum\limits_{j\ne i}(\bx_i^\top\bx_j)^2}{\min\limits_{i\in[n]}\norm{\bx_i}^4+(n-1)\min\limits_{i\in[n]}\frac{1}{n-1}\sum\limits_{j\ne i}(\bx_i^\top\bx_j)^2}.
\end{align*}

Then we only need to estimate $\norm{\bx_i}^4$ and $\frac{1}{n-1}\sum\limits_{j\ne i}(\bx_i^\top\bx_j)^2$ for each $i\in[n]$, respectively.

\paragraph{Step I: Estimation of $\norm{\bx_i}^4$.}

Let $ z_i=S^{-1/2}\bx_i$, then $\norm{\bx_i}^2= z_i^\top S z_i$ and $ z_1,\cdots, z_n\overset{{\rm i.i.d.}}{\sim}\cN( 0,I_d)$.

For a fix $i\in[n]$, by Lemma~\ref{lemma: hanson-wright}, there exists an absolute constant $C_1>0$ such that for any $\epsilon\in(0,1)$, we have
\begin{align*}
    \bbP\sbracket{\Big| z_i^\top S z_i-\Tr(S)\Big|\geq\epsilon\Tr(S)}
    \leq2\exp\sbracket{-C_1\min\bbracket{\frac{\epsilon^2\Tr^2(S)}{\norm{S}_{\rF}^2},\frac{\epsilon\Tr(S)}{\norm{S}_2}}}.
\end{align*}
Noticing that $\Tr(S)\norm{S}_2 = \lambda_1\sum_{i}\lambda_i\geq \sum_i \lambda_i^2=\|S\|_F$, we thus 
have 
\[
\frac{\Tr^2(S)}{\norm{S}_{\rF}^2}\geq\frac{\Tr(S)}{\norm{S}_2}=\srk(S).
\]
Therefore,
\begin{align*}
    \bbP\sbracket{\Big| z_i^\top S z_i-\Tr(S)\Big|\geq\epsilon\Tr(S)}
    \leq2\exp\sbracket{-C_1\frac{\Tr(S)}{\norm{S}_2}\min\bbracket{\epsilon,\epsilon^2}}=2\exp\sbracket{-C_1\epsilon^2\srk(S)}.
\end{align*}
Applying a union bound over all $i\in[n]$, we have
\begin{align*}
    \bbP\sbracket{\Big|\norm{\bx_i}^2-\Tr(S)\Big|\geq\epsilon\Tr(S),\forall i\in[n]}
    \leq2n\exp\sbracket{-C_1\epsilon^2\srk(S)}.
\end{align*}
In the other word, for any $\epsilon,\delta\in(0,1)$, if $\srk(S)\gtrsim\log(n)/\epsilon^2$, then \wp at least $1-\delta/3$, we have
\begin{align*}
    (1-\epsilon)^2\leq\frac{\norm{\bx_i}_2^4}{\Tr^2(S)}\leq(1+\epsilon)^2,\ \forall i\in[n].
\end{align*}

\paragraph{Step II: Estimation of $\frac{1}{n-1}\sum\limits_{j\ne i}(\bx_i^\top\bx_j)^2$.}

First, we fix $i\in[n]$. 
Notice that $(\bx_i^\top\bx_j)^2$ $(j\ne i)$ are not independent, so we need estimate by some decoupling tricks.

We denote $ z_i:=S^{-1/2}\bx_i$, then $ z_1,\cdots, z_n\overset{{\rm i.i.d.}}{\sim}\cN( 0,I_d)$ and $(\bx_i^\top\bx_j)^2=( z_i^\top S z_j)^2$.

For any fixed $\bv\in\bbS^{d-1}$, by Lemma~\ref{lemma: bernstien}, for any $\epsilon\in(0,1)$, we have
\begin{align*}
    &\bbP\sbracket{\Big|\frac{1}{n-1}\sum_{j\ne i}(\bv^\top z_j)^2-1\Big|\geq\epsilon}
    \\\leq&\bbP\sbracket{\Big|\frac{1}{n-1}\sum_{j\ne i}(\bv^\top z_j)^2-1\Big|\geq\epsilon}
    \leq2\exp\sbracket{-C_2(n-1)\epsilon^2},
\end{align*}
where $C_2>0$ is an absolute constant, independent of $\bv$ and $\epsilon$.

Then we have
\begin{align*}
    &\bbP\sbracket{\Big|\frac{1}{n-1}\sum_{j\ne i}(\bx_i^\top\bx_j)^2-\bx_i^\top S\bx_i\Big|\geq\epsilon\bx_i^\top S\bx_i}
    \\=&\bbP\sbracket{\Big|\frac{1}{n-1}\sum_{j\ne i}( z_i^\top S z_j)^2-\norm{S z_i}_2^2\Big|\geq\epsilon\norm{S z_i}_2^2}
    \\\overset{  q_i:=S z_i/\norm{S z_i}_2}{=}&\bbP\sbracket{\Big|\frac{1}{n-1}\sum_{j\ne i}(  q_i^\top z_j)^2-1\Big|\geq\epsilon}
    \\=&\bbE\mbracket{\bbI\bbracket{\Big|\frac{1}{n-1}\sum_{j\ne i}({  q}_i^\top z_j)^2-1\Big|\geq1}}
    \\=&\bbE_{q_i}\mbracket{\bbE\mbracket{\bbI\bbracket{\Big|\frac{1}{n-1}\sum_{j\ne i}({q}_i^\top z_j)^2-1\Big|\geq1}\Bigg|  q_i}}
    \\\leq&
    \bbE_{q_i}\mbracket{2\exp\sbracket{-C_2(n-1)\epsilon^2}}=2\exp\sbracket{-C_2(n-1)\epsilon^2}.
\end{align*}

Applying a union bound over all $i\in[n]$, we have
\begin{align*}
    \bbP\sbracket{\Big|\frac{1}{n-1}\sum_{j\ne i}(\bx_i^\top\bx_j)^2-\bx_i^\top S\bx_i\Big|\geq\epsilon\bx_i^\top S\bx_i,\forall i\in[n]}
    \leq2n\exp\sbracket{-C_2(n-1)\epsilon^2}.
\end{align*}
In the other word, for any $\epsilon,\delta\in(0,1)$, if $n/\log(n/\delta)\gtrsim1/\epsilon^2$, then \wp at least $1-\delta/3$, we have
\begin{align*}
    1-\epsilon\leq\frac{\frac{1}{n-1}\sum_{j\ne i}(\bx_i^\top\bx_j)^2}{\bx_i^\top S\bx_i}\leq1+\epsilon,\ \forall i\in[n].
\end{align*}

\paragraph{Step III: Estimation of $\bx_i^\top S\bx_i$.}

Let $ z_i=S^{-1/2}\bx_i$, then $\bx_i^\top S\bx_i= z_i^\top S^2 z_i$ and $ z_1,\cdots, z_n\overset{{\rm i.i.d.}}{\sim}\cN( 0,I_d)$.

In the same way as Step I(i), we obtain that: for any $\epsilon,\delta\in(0,1)$, if $\srk(S^2)\gtrsim\log(n)/\epsilon^2$, then \wp at least $1-\delta/3$, we have
\begin{align*}
    1-\epsilon\leq\frac{\bx_i^\top S\bx_i}{\Tr(S^2)}\leq1+\epsilon,\ \forall i\in[n].
\end{align*}

Combining our results in Step I, Step II, and Step III, we obtain the result for Linear Model: for any $\epsilon,\delta\in(0,1)$, if $n/\log(n/\delta)\gtrsim 1/\epsilon^2$ and $\min\{\srk(S),\srk(S^2)\}\gtrsim\log(n)/\epsilon^2$, then \wp at least $1-\delta/3-\delta/3-\delta/3=1-\delta$, we have
\begin{align*}
\mu( w)
\geq&\frac{(1-\epsilon)^2\Tr^2(S)+(n-1)(1-\epsilon)\min\limits_{i\in[n]}\bx_i^\top S\bx_i}{(1+\epsilon)^2\Tr^2(S)+(n-1)(1+\epsilon)\max\limits_{i\in[n]}\bx_i^\top S\bx_i}
\\\geq&\frac{(1-\epsilon)^2\Tr^2(S)+(n-1)(1-\epsilon)^2\Tr(S^2)}{(1+\epsilon)^2\Tr^2(S)+(n-1)(1+\epsilon)^2\Tr(S^2)}=\frac{(1-\epsilon)^2}{(1+\epsilon)^2};
\end{align*}
\begin{align*}
\mu(w)
\leq&\frac{(1+\epsilon)^2\Tr^2(S)+(n-1)(1+\epsilon)\max\limits_{i\in[n]}\bx_i^\top S\bx_i}{(1-\epsilon)^2\Tr^2(S)+(n-1)(1-\epsilon)\min\limits_{i\in[n]}\bx_i^\top S\bx_i}
\\\leq&\frac{(1+\epsilon)^2\Tr^2(S)+(n-1)(1+\epsilon)^2\Tr(S^2)}{(1-\epsilon)^2\Tr^2(S)+(n-1)(1-\epsilon)^2\Tr(S^2)}=\frac{(1+\epsilon)^2}{(1-\epsilon)^2}.
\end{align*}
From the arbitrary of $ w$, we obtain:
$$
\frac{(1-\epsilon)^2}{(1+\epsilon)^2}\leq\inf_{ w\in\bbR^d}\mu(w)\leq\sup_{ w\in\bbR^d}\mu(w)\leq\frac{(1+\epsilon)^2}{(1-\epsilon)^2}.
$$

\qed

\subsection{Proof of Theorem \ref{thm: weak: 2NN}}

For two-layer neural networks with fixed output layer, the gradient is
\begin{align*}
    \nabla f(\bx_i; \theta)=\sbracket{a_1\sigma'(  b_1^\top\bx_i)\bx_i^\top,\cdots,a_m\sigma'(  b_m^\top\bx_i)\bx_i^\top}^\top\in\bbR^{md}.
\end{align*}
For simplicity, denote $\nabla f_i( \theta):=\nabla f(\bx_i; \theta)$, $ u_i( \theta):=f_i( \theta)-f_i( \theta^*)$. Then we have:
\begin{align*}
\cL( \theta)=\frac{1}{2n}\sum_{i=1}^n u_i^2( \theta),\quad
 G( \theta)=\frac{1}{n}\sum_{i=1}^n\nabla f_i( \theta)\nabla f_i( \theta)^\top,\quad
\Sigma_1( \theta)=\frac{1}{n}\sum_{i=1}^nu_i^2( \theta)\nabla f_i( \theta)\nabla f_i( \theta)^\top.
\end{align*}

We have the following lower bound for $\mu(\theta)$:

\begin{align*}
    &\mu (\theta)
   =\frac{\Tr\sbracket{\sbracket{\frac{1}{n}\sum\limits_{i=1}^n\nabla f_i( \theta)\nabla f_i( \theta)^\top}\sbracket{\frac{1}{n}\sum\limits_{i=1}^nu_i^2( \theta)\nabla f_i( \theta)\nabla f_i( \theta)^\top}}}{\sbracket{\frac{1}{n}\sum\limits_{i=1}^nu_i^2( \theta)}\sbracket{\frac{1}{n^2}\sum\limits_{i=1}^n\sum\limits_{j=1}^n\sbracket{\nabla f_i( \theta)^\top\nabla f_i( \theta)}^2}}
   \\=&
   \frac{\frac{1}{n}\sum\limits_{i=1}^n u_i^2( \theta)\frac{1}{n}\sum\limits_{j=1}^n\sbracket{\nabla f_i( \theta)^\top\nabla f_j( \theta)}^2}{\sbracket{\frac{1}{n}\sum\limits_{i=1}^nu_i^2( \theta)}\sbracket{\frac{1}{n^2}\sum\limits_{i=1}^n\sum\limits_{j=1}^n\sbracket{\nabla f_i( \theta)^\top\nabla f_i( \theta)}^2}}
   \\\geq&\frac{\min\limits_{i\in[n]}\frac{1}{n}\sum\limits_{j=1}^n\sbracket{\nabla f_i( \theta)^\top\nabla f_j( \theta)}^2}{\frac{1}{n^2}\sum\limits_{i=1}^n\sum\limits_{i=1}^n\sbracket{\nabla f_i( \theta)^\top\nabla f_j( \theta)}^2}
   \geq\frac{\min\limits_{i\in[n]}\frac{1}{n}\sum\limits_{j=1}^n\sbracket{\alpha^2m\bx_i^\top\bx_j}^2}{\frac{1}{n^2}\sum\limits_{i=1}^n\sum\limits_{i=1}^n\sbracket{\beta^2m\bx_i^\top\bx_j}^2}=
   \frac{\alpha^2}{\beta^2}\frac{\min\limits_{i\in[n]}\frac{1}{n}\sum\limits_{j=1}^n\sbracket{\bx_i^\top\bx_j}^2}{\max\limits_{i\in[n]}\frac{1}{n}\sum\limits_{j=1}^n\sbracket{\bx_i^\top\bx_j}^2}.
\end{align*}

In the same way, the upper bound holds:
\begin{align*}
    \mu(\theta)\leq
    \frac{\beta^2}{\alpha^2}\frac{\max\limits_{i\in[n]}\frac{1}{n}\sum\limits_{j=1}^n\sbracket{\bx_i^\top\bx_j}^2}{\min\limits_{i\in[n]}\frac{1}{n}\sum\limits_{j=1}^n\sbracket{\bx_i^\top\bx_j}^2}.
\end{align*}

Notice that the terms $\frac{\min\limits_{i\in[n]}\frac{1}{n}\sum\limits_{j=1}^n\sbracket{\bx_i^\top\bx_j}^2}{\max\limits_{i\in[n]}\frac{1}{n}\sum\limits_{j=1}^n\sbracket{\bx_i^\top\bx_j}^2}$ and $\frac{\max\limits_{i\in[n]}\frac{1}{n}\sum\limits_{j=1}^n\sbracket{\bx_i^\top\bx_j}^2}{\min\limits_{i\in[n]}\frac{1}{n}\sum\limits_{j=1}^n\sbracket{\bx_i^\top\bx_j}^2}$ are independent of $ \theta$ and the same as the Linear Model.

Then repeating the same proof of Theorem~\ref{thm: weak: LM}, the result of this theorem differs from Linear Model by only the factor $\alpha^2/\beta^2$.
In other words, under the same condition with Linear Model, \wp at least $1-\delta$, we have
\begin{align*}
\frac{(1-\epsilon)^2}{(1+\epsilon)^2}\leq\inf_{ \theta\in\bbR^{md}}\mu(\theta)\leq\sup_{ \theta\in\bbR^{md}}\mu(\theta)\leq\frac{\beta^2}{\alpha^2}\frac{(1+\epsilon)^2}{(1-\epsilon)^2}.
\end{align*}
\qed

\section{Proofs of Section \ref{section: theory, strong align}}
\label{appendix: proof: strong align}

For the OLM $f(\bx; \theta)=F( \theta)^\top\bx$, 
let $ r( \theta)= F( \theta)- F( \theta^*)$ and $\nabla F(\theta)\in\RR^{d\times p}$ the Jacobian matrix. 

Then, we have for
the population loss: 
\begin{equation}\label{equ: proof: strong align: def population}
    \begin{aligned}
        \bar{G}(\theta)&=\bbE\Big[\nabla F( \theta)^\top\bx\bx^\top\nabla F( \theta)\Big] = \nabla F(\theta)^\top S\nabla F( \theta)\\
    \bar{\cL}( \theta)&=\frac{1}{2}\bbE\Big[\big( r(\theta)^\top \bx\big)^2\Big]=\frac{1}{2} r( \theta)^\top S r( \theta)\\ 
    \bar{\Sigma}_1(\theta)&=\bbE\Big[\big( r( \theta)^\top\bx\big)^2\nabla F^\top( \theta)\bx\bx^\top\nabla F( \theta)\Big]
    \end{aligned}
\end{equation}

\begin{lemma}[Proposition 2.3 in \cite{wu2022does}]\label{lemma: OLM equation Gaussian}
Let the data distribution be $\cN( 0, S)$. Then we have 
$$
\bar{\Sigma}_1(\theta)=2\bar{\cL}( \theta) \bar{G}(\theta) + \nabla\bar{\cL}( \theta)\nabla\bar{\cL}(\theta)^\top.
$$
\end{lemma}

\begin{lemma}\label{lemma: OLM nabla-L-v bound}
Under the same conditions in Lemma \ref{lemma: OLM equation Gaussian}, 
%if $ u( \theta)\ne 0$ and $\nabla F( \theta)\bv\ne0$, 
then we have
\begin{align*}
\big(\nabla\bar{\cL}(\theta)^\top\bv\big)^2\leq2\bar{\cL}( \theta)\bv^\top \bar{G}(\theta)\bv.
% \\&\left\|\nabla\mathcal{L}( \theta)\right\|^2\leq2\mathcal{L}( \theta)\lambda_{\max}( G( \theta)).
\end{align*}
\end{lemma}
\begin{proof}
For $a,b\in \RR^p$, define $\langle a,b\rangle_{S}=a^\top S b$ and $\|a\|_S=\sqrt{a^\top Sa}$. Since $S$ a positive semidefinite matrix, $\langle \cdot,\cdot\rangle_S$ is a well-defined inner product.  
Noticing that $\bar{\cL}(\theta)=\frac{1}{2} r( \theta)^\top S r(\theta)$, we have $\nabla\bar{\cL}(\theta)=\nabla F(\theta)^\top S r(\theta)$. Hence,
\begin{align*}
    &\big(\nabla\bar{\cL}( \theta)^\top\bv\big)^2
    =\bv^\top\nabla F(\theta)^\top S r( \theta) r( \theta)^\top S\nabla F( \theta)\bv=\left<\nabla F( \theta)\bv, r( \theta)\right>_{ S}^2
    \\\overset{(i)}{\leq}&\left\|\nabla F( \theta)\bv\right\|_{ S}^2\left\| r( \theta)\right\|_{ S}^2=2\bar{\cL}( \theta)\Big(\bv^\top\nabla F(\theta)^\top S\nabla F(\theta)\bv\Big)=2\bar{\cL}(\theta)\bv^\top \bar{G}(\theta)\bv,
\end{align*}
where $(i)$ follows from Cauchy-Schwarz inequality (see Lemma \ref{lemma: Cauchy-Schwarz}).

% For the second inequality, 
% \begin{align*}
%     &\left\|\nabla\mathcal{L}( \theta)\right\|^2= u( \theta)^\top S\nabla F( \theta)\nabla F( \theta)^\top S u( \theta)= u( \theta)^\top S^{1/2} S^{1/2}\nabla F( \theta)\nabla F( \theta)^\top S^{1/2} S^{1/2} u( \theta)
%     \\\leq& u( \theta)^\top S^{1/2} S^{1/2} u( \theta)\lambda_{\max}\Big( S^{1/2}\nabla F( \theta)\nabla F( \theta)^\top S^{1/2}\Big)=2\mathcal{L}( \theta)\lambda_{\max}\Big( S^{1/2}\nabla F( \theta)\nabla F( \theta)^\top S^{1/2}\Big)
%     \\=&2\mathcal{L}( \theta)\lambda_{\max}\Big(\nabla F( \theta)^\top S^{1/2} S^{1/2}\nabla F( \theta)\Big)=2\mathcal{L}( \theta)\lambda_{\max}( G( \theta)).
% \end{align*}

\end{proof}

\begin{lemma}\label{lemma: quadratic covarience concentrate}
Let $\bx_1,\cdots,\bx_n\overset{\rm i.i.d.}{\sim}\cN( 0,\bI_d)$. For any $\epsilon,\delta\in(0,1)$, if we choose $n\gtrsim\sbracket{d+\log(1/\delta)}/\epsilon^2$, then \wp at least $1-\delta$, we have:
\begin{align*}
    \sup\limits_{\bv\in\bbS^{d-1}}\left|\frac{1}{n}\sum_{i=1}^n(\bv^\top\bx_i)^2-1\right|\leq\epsilon.
\end{align*}
\end{lemma}

\begin{proof}
By Lemma \ref{lemma: covariance estimate sub Gaussian}, \wp at least $1-2\exp(-u)$, we have
\begin{align*}
    \left\|\frac{1}{n}\sum_{i=1}^n\bx_i\bx_i^\top-\bI_d\right\|\lesssim\sqrt{\frac{d+u}{n}}+\frac{d+u}{n}.
\end{align*}
Equivalently, we can rewrite this conclusion.
For any $\epsilon,\delta\in(0,1)$, if we choose $n\gtrsim\sbracket{d+\log(1/\delta)}/\epsilon^2$, then \wp at least $1-\delta$, we have:
\begin{align*}
    \sup\limits_{\bv\in\bbS^{d-1}}\left|\frac{1}{n}\sum_{i=1}^n(\bv^\top\bx_i)^2-1\right|
    \leq\left\|\frac{1}{n}\sum_{i=1}^n\bx_i\bx_i^\top-\bI_d\right\|
    \leq\epsilon.
\end{align*}
\end{proof}

% \begin{lemma}[Corollary 2 in~\citep{cai2022nearly}]\label{lemma: strong: n>d lower bound}
% Let $\bx_1,\cdots,\bx_n\overset{\rm i.i.d.}{\sim}\cN( 0,\bI_d)$. 
% There exists absolute constants $C_1,C_2,C_3>0$, such that if $n\geq C_3 d$, then \wp at least $1-\exp(-C_2 n)$, we have
% \begin{align*}
%     \inf\limits_{ u,\bv\in\bbS^{d-1}}\frac{1}{n}\sum_{i=1}^n(\bx_i^\top u)^2(\bx_i^\top\bv)^2\geq C_1.
% \end{align*}
% \end{lemma}

\begin{lemma}[Lemma 17 in~\citep{cai2022nearly}]\label{lemma: strong: 1-fix 1-arb bound}
    For any $0<\epsilon<1/2$, there are constants $C_1=C_1(\epsilon)>0$ and $C_2=C_2(\epsilon)>0$, such that if $n\geq C_1 d\log d$, then with probability at least $1-\frac{C_2}{n^2}$, it holds
    \begin{align*}
    \sup_{ u\in\bbS^{d-1}}\left|\frac{1}{n}\sum_{i=1}^n(\bx_i^\top u)^2(\bx_i^\top\bv)^2-\bbE\mbracket{(\bx^\top u)^2(\bx^\top\bv)^2}\right|\leq\epsilon.
    \end{align*}
\end{lemma}

% With the preparation of Lemma~\ref{lemma: quadratic covarience concentrate} and Lemma~\ref{lemma: strong: n>d lower bound}, now we give the proof of Theorem~\ref{thm: strong: lower bound}. 

\subsection{Proof of Theorem~\ref{thm: strong: OLM}}
We first need a few lemmas.
\begin{lemma}\label{lemma: standard Gaussian upper bound}
Let $ z_1,\cdots, z_n\overset{\rm i.i.d.}{\sim}\cN( 0,I_d)$. If $n\gtrsim d^2+\log^2(1/\delta)$, then \wp at least $1-\delta$, we have
\begin{align*}
    \sup_{\bv\in\bbS^{d-1}}\frac{1}{n}\sum_{i=1}^n( z_i^\top{\bv})^4\leq8.
\end{align*}
\end{lemma}

\begin{proof}
For $\bbS^{d-1}$, its covering number has the bound:
\[
\left(\frac{1}{\rho}\right)^d\leq\cN(\bbS^{d-1},\rho)\leq\left(\frac{2}{\rho}+1\right)^d,
\]
so there exist a $\rho$-net on $\bbS^{d-1}$: $\cV\subset\bbS^{d-1}$, s.t. $|\cV|\leq\left(\frac{2}{\rho}+1\right)^d$.

\paragraph*{Step I: Bounding the term on the $\rho$-net.}

For a fixed $\bv\in\cV$, 
due to $ z_i\overset{\rm i.i.d.}{\sim}\cN( 0,{I}_d)$, we can verify $( z_i^\top{\bv})^4$ is sub-Weibull random variable:
\begin{align*}
    &\bbE\exp\left(\left(( z_i^\top{\bv})^4\right)^{1/2}\right)
    =\bbE\exp\left(( z_i^\top{\bv})^2\right)\lesssim1,
\end{align*}
which means that there exist an absolute constant $C_1\geq1$ s.t. $\left\|( z_i^\top{\bv})^4\right\|_{\psi_{1/2}}\leq C_1$. 

% If $n>\log^3(1/\delta)$, then 
% $$\max\{\sqrt{\frac{\log(1/\delta)}{n}},\frac{\log^2(1/\delta)}{n}\}=\sqrt{\frac{\log(1/\delta)}{n}}.$$

By the concentration inequality for Sub-Weibull distribution with $\beta=1/2$ (Lemma \ref{lemma: concentration SW}) and $\bbE\Big[( y^\top{\bv})^4\Big]=3$, there exists an absolute constant $C_2\geq1$ s.t.
\begin{align*}
    \mathbb{P}\left(\left|\frac{1}{n}\sum_{i=1}^n\Big[( z_i^\top{\bv})^4\Big]-3\right|>\phi(n;\delta)\right)
    \leq2\delta,
\end{align*}
where $\phi(n;\delta)=C_2(\sqrt{\frac{\log(1/\delta)}{n}}+\frac{\log^2(1/\delta)}{n})$.
Applying a union bound over $\bv\in\cV$, we have:
\begin{align*}
    &\mathbb{P}\left(\exists\bv\in\cV s.t. \left|\frac{1}{n}\sum_{i=1}^n\Big[( z_i^\top{\bv})^4\Big]-3\right|>\phi(n;\delta)\right)
    \\
    \leq&
    \mathbb{P}\left(\bigcup\limits_{\bv\in\cV} \left\{\left|\frac{1}{n}\sum_{i=1}^n\Big[( z_i^\top{\bv})^4\Big]-3\right|>\phi(n;\delta)\right\}\right)
    \leq\sum\limits_{\bv\in\cV}\mathbb{P}\left(\left|\frac{1}{n}\sum_{i=1}^n\Big[( z_i^\top{\bv})^4\Big]-3\right|>\phi(n;\delta)\right)
    \\\leq&2|\cV|\exp\left(-\frac{n}{C_2^2}\right)=2\left(\frac{2}{\rho}+1\right)^d\delta.
\end{align*}

So \wp at least $1-2\left(\frac{2}{\rho}+1\right)^d\delta$, we have:
\[
\max\limits_{\bv\in\cV}\frac{1}{n}\sum_{i=1}^n\Big[( z_i^\top{\bv})^4\Big]\leq3+\phi(n;\delta).
\]

\paragraph{Step II: Estimate the error of the $\rho$-net approximation.}

For simplicity, we denote
\begin{align*}
    P:=\max\limits_{\bv\in\bbS^{d-1}}\frac{1}{n}\sum_{i=1}^n\Big[( z_i^\top{\bv})^4\Big],\quad 
    Q:=\max\limits_{\bv\in\cV}\frac{1}{n}\sum_{i=1}^n\Big[( z_i^\top{\bv})^4\Big].
\end{align*}

Let $\bv\in\bbS^{d-1}$ such that $\frac{1}{n}\sum_{i=1}^n\Big[( z_i^\top{\bv})^4\Big]=P$, then there exist $\bv_0\in\cV$, s.t. $\left\|\bv-\bv_0\right\|\leq\rho$. 

On the one hand,
\begin{align*}
    &\left|\frac{1}{n}\sum_{i=1}^n( z_i^\top{\bv})^4-\frac{1}{n}\sum_{i=1}^n( z_i^\top{\bv_0})^4\right|
    =\left|\frac{1}{n}\sum_{i=1}^n\Big(( z_i^\top{\bv})^4-( z_i^\top{\bv_0})^4\Big)\right|
    \\=&
    \left|\frac{1}{n}\sum_{i=1}^n\sbracket{ z_i^\top(\bv-\bv_0)}\sbracket{ z_i^\top(\bv+\bv_0)}\sbracket{( z_i^\top\bv)^2+( z_i^\top\bv_0)^2}\right|
    \\\leq&
    \left|\frac{1}{n}\sum_{i=1}^n\sbracket{ z_i^\top(\bv-\bv_0)}\sbracket{ z_i^\top(\bv+\bv_0)}( z_i^\top\bv)^2\right|+\left|\frac{1}{n}\sum_{i=1}^n\sbracket{ z_i^\top(\bv-\bv_0)}\sbracket{ z_i^\top(\bv+\bv_0)}( z_i^\top\bv_0)^2\right|
    \\\leq&
    \sqrt{\frac{1}{n}\sum_{i=1}^n\sbracket{ z_i^\top(\bv-\bv_0)}^2\sbracket{ z_i^\top(\bv+\bv_0)}^2}\sbracket{\sqrt{\frac{1}{n}\sum_{i=1}^n( z_i^\top\bv)^4}+\sqrt{\frac{1}{n}\sum_{i=1}^n( z_i^\top\bv_0)^4}}
    \\\leq&\sqrt[4]{\frac{1}{n}\sum_{i=1}^n\sbracket{ z_i^\top(\bv-\bv_0)}^4}\sqrt[4]{\frac{1}{n}\sum_{i=1}^n\sbracket{ z_i^\top(\bv+\bv_0)}^4}\sbracket{\sqrt{\frac{1}{n}\sum_{i=1}^n( z_i^\top\bv)^4}+\sqrt{\frac{1}{n}\sum_{i=1}^n( z_i^\top\bv_0)^4}}
    \\\leq&\norm{\bv-\bv_0}P^{1/4}\norm{\bv+\bv_0}P^{1/4}(\sqrt{P}+\sqrt{Q})\leq2\rho\sqrt{P}(\sqrt{P}+\sqrt{Q})
\end{align*}

On the other hand, 
\begin{align*}
    \left|\frac{1}{n}\sum_{i=1}^n( z_i^\top{\bv})^4-\frac{1}{n}\sum_{i=1}^n( z_i^\top{\bv_0})^4\right|\geq P-\sum_{i=1}^n( z_i^\top{\bv_0})^4\geq P- Q.
\end{align*}

Hence, we obtain
\begin{align*}
    P-Q\leq 2\rho\sqrt{P}(\sqrt{P}+\sqrt{Q}),
\end{align*}
which means that
\begin{align*}
    P\leq\sbracket{\frac{1}{1-2\rho}}^2 Q.
\end{align*}

% It is easy to verify that $\left\| z_i\right\|^4$ is an sub-Weibull random variable s.t. $\left\|\left\| z_i\right\|^4\right\|_{\psi_{1/2}}\leq C_3d^2$. And
% \[\bbE\Big[\left\| y\right\|^4\Big]=\bbE\left(\sum_{i=1}^d z_i^2\right)^2=\sum_{i=1}^d\bbEz_i^4+2\sum_{i<j}\bbEz_i^2\bbEz_j^2=3d+2\frac{d(d-1)}{2}=d(d+2).
% \]
% By the concentration inequality for Sub-Weibull distribution with $\beta=1/2$ (Lemma \ref{lemma: concentration SW}), there exists an absolute constant $C_4\geq1$ s.t.
% \begin{align*}
% \mathbb{P}\left(\left|\frac{1}{n}\sum_{i=1}^n\left\| z_i\right\|^4-d(d+2)\right|>d(d+2)\right)\leq2\exp\left(- n\min\Big\{\frac{d(d+2)}{C_4 d^2},\frac{(d(d+2))^2}{C_4^2d^4}\Big\}\right)
% \leq2\exp\left(-\frac{2n}{C_4^2}\right).
% \end{align*}
% Then \wp at least $1-2\exp\left(-\frac{2n}{C_4^2}\right)$, we have:
% \[
% \frac{1}{n}\sum_{i=1}^n\left\| z_i\right\|^4\leq2d(d+2),
% \]
% which means:
% \begin{align*}
%     \left|\frac{1}{n}\sum_{i=1}^n\left\| z_i\right\|^2( z_i^\top{\bv})^2-\frac{1}{n}\sum_{i=1}^n\left\| z_i\right\|^2( z_i^\top{\bv_0})^2\right|\leq4\rho d(d+2).
% \end{align*}
% So we have:
% \begin{align*}
%     \sup_{\bv\in\bbS^{d-1}}\min_{\bv_0\in\cV}\left|\frac{1}{n}\sum_{i=1}^n\left\| z_i\right\|^2( z_i^\top{\bv})^2-\frac{1}{n}\sum_{i=1}^n\left\| z_i\right\|^2( z_i^\top{\bv_0})^2\right|\leq4\rho d(d+2).
% \end{align*}

\paragraph{Step III: The bound for any $\bv\in\bbS^{d-1}$.}

Select $\rho=\frac{1}{2}(1-\frac{1}{\sqrt{2}})$ and denote $\delta'=2(\frac{2}{\rho}+1)^d\delta$. And we choose $n\gtrsim d^2+\log^2(1/\delta')$, which ensures $\phi(n;\delta)\leq1$.

Then combining the results in Step I and Step II, we know that: \wp at least $1-\delta'$, we have:
\[
\max\limits_{\bv\in\cV}\frac{1}{n}\sum_{i=1}^n\Big[( z_i^\top{\bv})^4\Big]\leq3+1=4;\quad 
\max\limits_{\bv\in\bbS^{d-1}}\frac{1}{n}\sum_{i=1}^n\Big[( z_i^\top{\bv})^4\Big]\leq2\max\limits_{\bv\in\cV}\frac{1}{n}\sum_{i=1}^n\Big[( z_i^\top{\bv})^4\Big],
\]
which means
\begin{align*}
    \max\limits_{\bv\in\bbS^{d-1}}\frac{1}{n}\sum_{i=1}^n\Big[( z_i^\top{\bv})^4\Big]\leq2\cdot 4=8.
\end{align*}

\end{proof}

\begin{lemma}\label{lemma: strong: d direction}
Let $\bx_1,\cdots,\bx_n\overset{\rm i.i.d.}{\sim}\cN( 0,\bI_d)$. For any $\epsilon,\delta\in(0,1)$, if we choose
\begin{align*}   n\gtrsim\max\bbracket{\sbracket{d^2\log^2\sbracket{1/{\epsilon}}+\log^2(1/\delta)}/\epsilon, \sbracket{d\log\sbracket{1/\epsilon}+\log(1/\delta)}/\epsilon^2},
\end{align*}
then \wp at least $1-\delta$, we have:
\begin{align*}
    \sup\limits_{ w,\bv\in\bbS^{d-1}}\left|\frac{1}{n}\sum_{i=1}^n( w^\top\bx_i)^2(\bv^\top\bx_i)^2-\bbE\Big[( w^\top\bx_1)^2(\bv^\top\bx_1)^2\Big]\right|\leq\epsilon.
\end{align*}
\end{lemma}

\begin{proof}
For $\bbS^{d-1}$, its covering number has the bound:
\[
\left(\frac{1}{\rho}\right)^d\leq\cN(\bbS^{d-1},\rho)\leq\left(\frac{2}{\rho}+1\right)^d,
\]

so there exist two $\rho$-nets on $\bbS^{d-1}$: $\cW\subset\bbS^{d-1}$ and $\cV\subset\bbS^{d-1}$, s.t.
\[
|\cW|\leq\left(\frac{2}{\rho}+1\right)^d,\quad
|\cV|\leq\left(\frac{2}{\rho}+1\right)^d.
\]

\paragraph{Step I: Bounding the term on the $\rho$-net.}

In this step, will estimate the term 
\[
\left|\frac{1}{n}\sum_{i=1}^n( w^\top\bx_i)^2(\bv^\top\bx_i)^2-\bbE\Big[( w^\top\bx)^2(\bv^\top\bx)^2\Big]\right| 
\]
for any $ w\in\cW$ and $\bv\in\cV$.

For fixed $ w\in\cW$ and $\bv\in\cV$, we denote $X_i^{ w,\bv}:=( w^\top\bx_i)^2(\bv^\top\bx_i)^2$. We can verify $X_i$ is a sub-Weibull random variable with $\beta=1/2$ (Definition \ref{def: sub-Weibull}):
\begin{align*}
    &\bbE\Bigg[\exp\Big(\left|( w^\top\bx_i)^2(\bv^\top\bx_i)\right|^{1/2}\Big)\Bigg]=
    \bbE\Bigg[\exp\Big(| w^\top\bx_i||\bv^\top\bx_i|\Big)\Bigg]
    \\\leq&
    \bbE\Bigg[\exp\Bigg(\frac{( w^\top\bx_i)^2+(\bv^\top\bx_i)^2}{2}\Bigg)\Bigg]
    =\bbE\Bigg[\exp\Big(\frac{( w^\top\bx_i)^2}{2}\Big)\exp\Big(\frac{(\bv^\top\bx_i)^2}{2}\Big)\Bigg]
    \\\overset{\text{Lemma \ref{lemma: Cauchy-Schwarz}}}{\leq}&\sqrt{\bbE\Bigg[\exp\Big(( w^\top\bx_i)^2}\Big)\cdot\sqrt{\bbE\Bigg[\exp\Big((\bv^\top\bx_i)^2\Big)\Bigg]}
    \overset{\left\|(\bv^\top\bx_i)^2\right\|_{\psi_1}\leq C_3}{\lesssim}1,
\end{align*}
which means that there exists an absolute constant $C_4\geq1$, s.t. $\left\|X_i^{ w,\bv}\right\|_{\psi_{1/2}}\leq C_4$. By the concentration inequality for Sub-Weibull distribution with $\beta=1/2$ (Lemma \ref{lemma: concentration SW}), there exists an absolute constant $C_5\geq1$, s.t.
\begin{align*}
\mathbb{P}\left(\left|\frac{1}{n}\sum_{i=1}^n X_i^{ w,\bv}-\frac{1}{n}\sum_{i=1}^n\bbE\big[X_i^{ w,\bv}\big]\right|>\psi(n;\delta)\right)\leq\delta.
\end{align*}
where $\psi(n;\delta)=C_5\sbracket{\sqrt{\frac{\log(1/\delta)}{n}}+\frac{(\log(1/\delta))^2}{n}}$.

Applying an union bound over $ w\in\cW$ and $\bv\in\cV$, we have:
\begin{align*}
    &\mathbb{P}\left(\exists  w\in\cW,\bv\in\cV,{\rm s.t.}\left|\frac{1}{n}\sum_{i=1}^n X_i^{ w,\bv}-\frac{1}{n}\sum_{i=1}^n\bbE\big[X_i^{ w,\bv}\big]\right|>\psi(n;\delta)\right)
    \\\leq&\mathbb{P}\left(\bigcup\limits_{( w,\bv)\in\cW\times\cV}\left\{\exists  w\in\cW,\bv\in\cV,{\rm s.t.}\left|\frac{1}{n}\sum_{i=1}^n X_i^{ w,\bv}-\frac{1}{n}\sum_{i=1}^n\bbE\big[X_i^{ w,\bv}\big]\right|>\psi(n;\delta)\right\}\right)
    \\\leq&
    \sum\limits_{( w,\bv)\in\cW\times\cV}\mathbb{P}\left(\exists  w\in\cW,\bv\in\cV,{\rm s.t.}\left|\frac{1}{n}\sum_{i=1}^n X_i^{ w,\bv}-\frac{1}{n}\sum_{i=1}^n\bbE\big[X_i^{ w,\bv}\big]\right|>\psi(n;\delta)\right)
    \\\leq&
    2|\cW||\cV|\delta
    \leq2\left(\frac{2}{\rho}+1\right)^{2d}\delta.
\end{align*}

So \wp at least  $1-2\left(\frac{2}{\rho}+1\right)^{2d}\delta$, we have:
\[
\sup\limits_{ w\in\cW,\bv\in\cV}\left|\frac{1}{n}\sum_{i=1}^n( w^\top\bx_i)^2(\bv^\top\bx_i)^2-\bbE\Big[( w^\top\bx)^2(\bv^\top\bx)^2\Big]\right|\leq\psi(n;\delta).
\]

\paragraph{Step II: Estimate the population error of the $\rho$-net approximation.}

Let $ w,\bv, w_0,\bv_0\in\bbS^{d-1}$, s.t. $\left\| w- w_0\right\|\leq\rho$ and $\left\|\bv-\bv_0\right\|\leq\rho$. For the population error, we have
\begin{align*}
    &\left|\bbE\Big[( w^\top\bx)^2(\bv^\top\bx)^2\Big]-\bbE\Big[( w_0^\top\bx)^2(\bv_0^\top\bx)^2\Big]\right|
    \\=&
    \left|\bbE\Big[\Big(( w^\top\bx)^2-( w_0^\top\bx)^2\Big)(\bv^\top\bx)^2\Big]+\bbE\Big[( w_0^\top\bx)^2\Big((\bv^\top\bx)^2-(\bv_0^\top\bx)^2\Big)\Big]\right|
    \\\leq&
    \left|\bbE\Big[\Big(( w^\top\bx)^2-( w_0^\top\bx)^2\Big)(\bv^\top\bx)^2\Big]\right|+\left|\bbE\Big[( w_0^\top\bx)^2\Big((\bv^\top\bx)^2-(\bv_0^\top\bx)^2\Big)\Big]\right|
\end{align*}

We first bound $\left|\bbE\Big[\Big(( w^\top\bx)^2-( w_0^\top\bx)^2\Big)(\bv^\top\bx)^2\Big]\right|$:
\begin{align*}
    &\left|\bbE\Big[\Big(( w^\top\bx)^2-( w_0^\top\bx)^2\Big)(\bv^\top\bx)^2\Big]\right|=\left|\bbE\Big[\Big(( w- w_0)^\top\bx\bx^\top( w+ w_0)(\bv^\top\bx)^2\Big]\right|
    \\\leq&
    \sbracket{\bbE\mbracket{\sbracket{( w- w_0)^\top\bx\bx^\top( w+ w_0)}^2}}^{1/2}\sbracket{\bbE\mbracket{(\bv^\top\bx)^4}}^{1/2}
    \\\leq&
    \sbracket{\bbE\mbracket{\sbracket{( w- w_0)^\top\bx}^4}}^{1/4}\sbracket{\bbE\mbracket{\sbracket{( w+ w_0)^\top\bx}^4}}^{1/4}\sbracket{\bbE\mbracket{(\bv^\top\bx)^4}}^{1/2}
    \\\leq&3\norm{( w- w_0)}\norm{( w+ w_0)}\norm{\bv}^2\leq6\rho.
\end{align*}

Repeating the proof above, we also have:
\begin{align*}
\left|\bbE\Big[\Big(( w^\top\bx)^2-( w_0^\top\bx)^2\Big)(\bv^\top\bx)^2\Big]\right|
    \leq6\rho.
\end{align*}
Combining these two inequalities, we have:
\begin{align*}
    \left|\bbE\Big[( w^\top\bx)^2(\bv^\top\bx)^2\Big]-\bbE\Big[( w_0^\top\bx)^2(\bv_0^\top\bx)^2\Big]\right|\leq
    6\rho+6\rho =12\rho.
\end{align*}
Due to the arbitrariness of $ w,\bv, w_0,\bv_0$, we obtain
\begin{align*}
    \sup_{\substack{ w,\bv, w_0,\bv_0\in\bbS^{d-1}\\\left\| w- w_0\right\|\leq\rho, \left\|\bv-\bv_0\right\|\leq\rho}}\left|\bbE\Big[( w^\top\bx)^2(\bv^\top\bx)^2\Big]-\bbE\Big[( w_0^\top\bx)^2(\bv_0^\top\bx)^2\Big]\right|\leq12\rho.
\end{align*}

\paragraph{Step III: Estimate the empirical error of the $\rho$-net approximation.}

Let $ w,\bv, w_0,\bv_0\in\bbS^{d-1}$, s.t. $\left\| w- w_0\right\|\leq\rho$ and $\left\|\bv-\bv_0\right\|\leq\rho$. For the empirical error, we have
\begin{align*}
    &\left|\frac{1}{n}\sum_{i=1}^n( w^\top\bx_i)^2(\bv^\top\bx_i)^2-\frac{1}{n}\sum_{i=1}^n( w_0^\top\bx_i)^2(\bv_0^\top\bx_i)^2\right|
    \\=&
    \left|\frac{1}{n}\sum_{i=1}^n\Big[\Big(( w^\top\bx_i)^2-( w_0^\top\bx_i)^2\Big)(\bv^\top\bx_i)^2\Big]+\frac{1}{n}\sum_{i=1}^n\Big[( w_0^\top\bx_i)^2\Big((\bv^\top\bx_i)^2-(\bv_0^\top\bx_i)^2\Big)\Big]\right|
    \\\leq&
    \left|\frac{1}{n}\sum_{i=1}^n\Big[\Big(( w^\top\bx_i)^2-( w_0^\top\bx_i)^2\Big)(\bv^\top\bx_i)^2\Big]\right|+\left|\frac{1}{n}\sum_{i=1}^n\Big[( w_0^\top\bx_i)^2\Big((\bv^\top\bx_i)^2-(\bv_0^\top\bx_i)^2\Big)\Big]\right|
\end{align*}
We first bound $\left|\frac{1}{n}\sum_{i=1}^n\Big[\Big(( w^\top\bx_i)^2-( w_0^\top\bx_i)^2\Big)(\bv^\top\bx_i)^2\Big]\right|$:
\begin{align*}
    \left|\frac{1}{n}\sum_{i=1}^n\Big[\Big(( w^\top\bx_i)^2-( w_0^\top\bx_i)^2\Big)(\bv^\top\bx_i)^2\Big]\right|
    =&\left|\frac{1}{n}\sum_{i=1}^n\Big[\Big(( w- w_0)^\top\bx_i\bx_i^\top( w+ w_0)(\bv^\top\bx_i)^2\Big]\right|
    \\\leq&2\rho\sup\limits_{ u\in\bbS^{d-1}}\frac{1}{n}\sum_{i=1}^n(\bx_i^\top{ u})^4.
\end{align*}

Repeating the proof above, we also have $\left|\frac{1}{n}\sum_{i=1}^n\Big[( w_0^\top\bx_i)^2\Big((\bv^\top\bx_i)^2-(\bv_0^\top\bx_i)^2\Big)\Big]\right|\leq2\rho\sup\limits_{ u\in\bbS^{d-1}}\frac{1}{n}\sum_{i=1}^n(\bx_i^\top{ u})^4$. Combining these two bounds, we have:
\begin{align*}
    &\left|\frac{1}{n}\sum_{i=1}^n( w^\top\bx_i)^2(\bv^\top\bx_i)^2-\frac{1}{n}\sum_{i=1}^n( w_0^\top\bx_i)^2(\bv_0^\top\bx_i)^2\right|
    \leq4\rho\sup\limits_{ u\in\bbS^{d-1}}\frac{1}{n}\sum_{i=1}^n(\bx_i^\top{ u})^4.
\end{align*}

Using Lemma \ref{lemma: standard Gaussian upper bound}, if $n\gtrsim d^2+\log^2(1/\delta')$, then \wp at least $1-\delta'/2$, we have
$\sup\limits_{ u\in\bbS^{d-1}}\frac{1}{n}\sum_{i=1}^n(\bx_i^\top{ u})^4\leq8$. 

Hence, \wp at least $1-\delta'/2$, we have
\begin{align*}
    \left|\frac{1}{n}\sum_{i=1}^n( w^\top\bx_i)^2(\bv^\top\bx_i)^2-\frac{1}{n}\sum_{i=1}^n( w_0^\top\bx_i)^2(\bv_0^\top\bx_i)^2\right|
    \leq32\rho.
\end{align*}
Due to the arbitrariness of $ w,\bv, w_0,\bv_0$, we obtain
\begin{align*}
    \sup_{\substack{ w,\bv, w_0,\bv_0\in\bbS^{d-1}\\\left\| w- w_0\right\|\leq\rho, \left\|\bv-\bv_0\right\|\leq\rho}}\left|\frac{1}{n}\sum_{i=1}^n( w^\top\bx_i)^2(\bv^\top\bx_i)^2-\frac{1}{n}\sum_{i=1}^n( w_0^\top\bx_i)^2(\bv_0^\top\bx_i)^2\right|\leq 32\rho.
\end{align*}

\paragraph{Step IV: The bound for any $ w,\bv\in\bbS^{d-1}$.}

Combining the results in Step I, II, and II, we know that \wp at least 
$1-\frac{\delta'}{2}-(\frac{2}{\rho}+1)^d$, we have
\begin{gather*}
\sup\limits_{ w\in\cW,\bv\in\cV}\left|\frac{1}{n}\sum_{i=1}^n( w^\top\bx_i)^2(\bv^\top\bx_i)^2-\bbE\Big[( w^\top\bx)^2(\bv^\top\bx)^2\Big]\right|\leq\psi(n;\delta),
\\
 \sup_{\substack{ w,\bv, w_0,\bv_0\in\bbS^{d-1}\\\left\| w- w_0\right\|\leq\rho, \left\|\bv-\bv_0\right\|\leq\rho}}\left|\bbE\Big[( w^\top\bx)^2(\bv^\top\bx)^2\Big]-\bbE\Big[( w_0^\top\bx)^2(\bv_0^\top\bx)^2\Big]\right|\leq12\rho,
\\
 \sup_{\substack{ w,\bv, w_0,\bv_0\in\bbS^{d-1}\\\left\| w- w_0\right\|\leq\rho, \left\|\bv-\bv_0\right\|\leq\rho}}\left|\frac{1}{n}\sum_{i=1}^n( w^\top\bx_i)^2(\bv^\top\bx_i)^2-\frac{1}{n}\sum_{i=1}^n( w_0^\top\bx_i)^2(\bv_0^\top\bx_i)^2\right|\leq 32\rho.
\end{gather*}
Then for any $ w,\bv\in\bbS^{d-1}$, there exists $ w_0\in\cW,\bv_0\in\cV$ s.t. $\left\| w- w_0\right\|\leq\rho$ and $\left\|\bv-\bv_0\right\|\leq\rho$, so
\begin{align*}
&\left|\frac{1}{n}\sum_{i=1}^n( w^\top\bx_i)^2(\bv^\top\bx_i)^2-\bbE\Big[( w^\top\bx)^2(\bv^\top\bx)^2\Big]\right|
\\=&
\Bigg|\frac{1}{n}\sum_{i=1}^n( w^\top\bx_i)^2(\bv^\top\bx_i)^2-\frac{1}{n}\sum_{i=1}^n( w_0^\top\bx_i)^2(\bv_0^\top\bx_i)^2
+\frac{1}{n}\sum_{i=1}^n( w_0^\top\bx_i)^2(\bv_0^\top\bx_i)^2
\\&-\bbE\Big[( w_0^\top\bx)^2(\bv_0^\top\bx)^2\Big]+\bbE\Big[( w_0^\top\bx)^2(\bv_0^\top\bx)^2\Big]-\bbE\Big[( w^\top\bx)^2(\bv^\top\bx)^2\Big]\Bigg|
\\\leq&\Bigg|\frac{1}{n}\sum_{i=1}^n( w^\top\bx_i)^2(\bv^\top\bx_i)^2-\frac{1}{n}\sum_{i=1}^n( w_0^\top\bx_i)^2(\bv_0^\top\bx_i)^2\Bigg|
\\&+\Bigg|\frac{1}{n}\sum_{i=1}^n( w_0^\top\bx_i)^2(\bv_0^\top\bx_i)^2-\bbE\Big[( w_0^\top\bx)^2(\bv_0^\top\bx)^2\Big]\Bigg|+\Bigg|\bbE\Big[( w_0^\top\bx)^2(\bv_0^\top\bx)^2\Big]-\bbE\Big[( w^\top\bx)^2(\bv^\top\bx)^2\Big]\Bigg|
\\\leq&\sup_{\substack{ w,\bv, w_0,\bv_0\in\bbS^{d-1}\\\left\| w- w_0\right\|\leq\rho, \left\|\bv-\bv_0\right\|\leq\rho}}\left|\frac{1}{n}\sum_{i=1}^n( w^\top\bx_i)^2(\bv^\top\bx_i)^2-\frac{1}{n}\sum_{i=1}^n( w_0^\top\bx_i)^2(\bv_0^\top\bx_i)^2\right|
\\&+
\sup\limits_{ w\in\cW,\bv\in\cV}\left|\frac{1}{n}\sum_{i=1}^n( w^\top\bx_i)^2(\bv^\top\bx_i)^2-\bbE\Big[( w^\top\bx)^2(\bv^\top\bx)^2\Big]\right|
\\
&+\sup_{\substack{ w,\bv, w_0,\bv_0\in\bbS^{d-1}\\\left\| w- w_0\right\|\leq\rho, \left\|\bv-\bv_0\right\|\leq\rho}}\left|\bbE\Big[( w^\top\bx)^2(\bv^\top\bx)^2\Big]-\bbE\Big[( w_0^\top\bx)^2(\bv_0^\top\bx)^2\Big]\right|
\\\leq&
32\rho+\psi(n;\delta)+12\rho
=44\rho+\psi(n;\delta).
\end{align*}
Due to the arbitrariness of $ w,\bv$, we have
\[
\sup_{ w,\bv\in\bbS^{d-1}}\left|\frac{1}{n}\sum_{i=1}^n( w^\top\bx_i)^2(\bv^\top\bx_i)^2-\bbE\Big[( w^\top\bx)^2(\bv^\top\bx)^2\Big]\right|\leq
44\rho+\psi(n;\delta)
\]

Select $\rho=\frac{\epsilon}{66}$ and $\delta'/2=2(1+\frac{2}{\rho})^{2d}\delta$. And we choose
\begin{align*}
  n\gtrsim\max\bbracket{\sbracket{d^2\log^2\sbracket{1/{\epsilon}}+\log^2(1/\delta)}/\epsilon,\sbracket{d\log\sbracket{1/{\epsilon}}+\log(1/\delta)}/\epsilon^2},
\end{align*}
which satisfies $\psi(n;\delta)\leq\epsilon/3$.

Then \wp at least $1-\delta'/2-\delta'/2=1-\delta'$, we have
\[
\sup_{ w,\bv\in\bbS^{d-1}}\left|\frac{1}{n}\sum_{i=1}^n( w^\top\bx_i)^2(\bv^\top\bx_i)^2-\bbE\Big[( w^\top\bx)^2(\bv^\top\bx)^2\Big]\right|
\leq\frac{44}{66}\epsilon+\frac{1}{3}\epsilon=\epsilon.
\]

\end{proof}

With the preparation of Lemma~\ref{lemma: OLM equation Gaussian},~\ref{lemma: quadratic covarience concentrate}, and~\ref{lemma: strong: d direction}, now we give the proof of Theorem~\ref{thm: strong: OLM}.

\paragraph*{Proof of Theorem \ref{thm: strong: OLM}.}
Let $ z_i= S^{-1/2}\bx_i$. Then $ z_1,\cdots, z_n\overset{{\rm i.i.d.}}{\sim}\cN( 0,I_d).$
\begin{align*}
g( \theta;\bv)
=&
% \frac{\bv^\top\hat{\Sigma}( \theta)\bv}{2\hat{\mathcal{L}}( \theta)\bv^\top\hat{ G}( \theta)\bv}
\frac{\frac{1}{n}\sum\limits_{i=1}^n\Big( r( \theta)^\top\bx_i\Big)^2\Big(\big(\nabla F( \theta)\bv\big)^\top\bx_i\Big)^2}{\frac{1}{n}\sum\limits_{i=1}^n\Big( r( \theta)^\top\bx_i\Big)^2\cdot\frac{1}{n}\sum\limits_{i=1}^n\Big(\big(\nabla F( \theta)\bv\big)^\top\bx_i\Big)^2}
\\=&\frac{\frac{1}{n}\sum\limits_{i=1}^n\Big(( S^{1/2} r( \theta))^\top z_i\Big)^2\Big(\big( S^{1/2}\nabla F( \theta)\bv\big)^\top z_i\Big)^2}{\frac{1}{n}\sum\limits_{i=1}^n\Big(( S^{1/2} r( \theta))^\top z_i\Big)^2\cdot\frac{1}{n}\sum\limits_{i=1}^n\Big(\big( S^{1/2}\nabla F( \theta)\bv\big)^\top z_i\Big)^2},
\end{align*}

Case (i). If $ S^{1/2} r( \theta)= 0$ or $ S^{1/2}\nabla F( \theta)\bv=0$, we have $g( \theta;\bv)=\frac{0}{0}=1$, this theorem holds.

Case (ii). If $ S^{1/2} r( \theta)\ne 0$ and $ S^{1/2}\nabla F( \theta)\bv\ne0$, we define the following normalized vectors:
\[
\tilde{ r}( \theta):=\frac{ S^{1/2} r( \theta)}{\left\| S^{1/2} r( \theta)\right\|}\in\bbS^{d-1}\quad \tilde{ w}( \theta;\bv):=\frac{ S^{1/2}\nabla F( \theta)\bv}{\left\| S^{1/2}\nabla F( \theta)\bv\right\|}\in\bbS^{d-1}.
\]
From the homogeneity of $g( \theta;\bv)$, we have:
\[
g( \theta;\bv)=\frac{\frac{1}{n}\sum\limits_{i=1}^n\Big(\tilde{ r}( \theta)^\top z_i\Big)^2\Big(\tilde{ w}( \theta;\bv)^\top z_i\Big)^2}{\frac{1}{n}\sum\limits_{i=1}^n\Big(\tilde{ r}( \theta)^\top z_i\Big)^2\cdot\frac{1}{n}\sum\limits_{i=1}^n\Big(\tilde{ w}( \theta;\bv)^\top z_i\Big)^2}.
\]

By Lemma~\ref{lemma: quadratic covarience concentrate} and~\ref{lemma: strong: d direction}, for any $\epsilon,\delta\in(0,1)$, if we choose
\begin{align*}   n\gtrsim\max\bbracket{\sbracket{d^2\log^2\sbracket{1/{\epsilon}}+\log^2(1/\delta)}/\epsilon, \sbracket{d\log\sbracket{1/\epsilon}+\log(1/\delta)}/\epsilon^2},
\end{align*}
then \wp at least $1-\delta$, the following inequalities hold:
\begin{gather*}
\sup\limits_{\bv\in\bbS^{d-1}}\left|\frac{1}{n}\sum_{i=1}^n(\bv^\top z_i)^2-1\right|\leq\epsilon,
\\
\sup\limits_{ w,\bv\in\bbS^{d-1}}\left|\frac{1}{n}\sum_{i=1}^n( w^\top z_i)^2(\bv^\top z_i)^2-\bbE\Big[( w^\top z_1)^2(\bv^\top z_1)^2\Big]\right|\leq\epsilon;
\end{gather*}
These imply that for any $ \theta,\bv\in\bbR^p$, we have:
\begin{equation}\label{equ of proof: thm: strong: arb dir:: emp-pop}
\frac{\bbE\Big[(\tilde{ r}( \theta)^\top y)^2(\tilde{ w}( \theta;\bv)^\top y)^2\Big]-\epsilon}{(1+\epsilon)^2}\leq g( \theta;\bv)\leq\frac{\bbE\Big[(\tilde{ r}( \theta)^\top z_1)^2(\tilde{ w}( \theta;\bv)^\top z_1)^2\Big]+\epsilon}{(1-\epsilon)^2}.
\end{equation}

First, we derive the upper bound for~\eqref{equ of proof: thm: strong: arb dir:: emp-pop}:
\begin{align*}
    {\rm RHS}=&\frac{\epsilon}{(1-\epsilon)^2}+\frac{\bbE\Big[(\tilde{ r}( \theta)^\top y)^2(\tilde{ w}( \theta;\bv)^\top y)^2\Big]}{(1-\epsilon)^2\Big(\tilde{ r}( \theta)^\top\tilde{ r}( \theta)\Big)\Big(\tilde{ w}( \theta;\bv)^\top\tilde{ w}( \theta;\bv)\Big)}
    \\\overset{\text{Homogeneity}}{=}&
    \frac{\epsilon}{(1-\epsilon)^2}+\frac{\bbE\Big[(( S^{1/2} r( \theta))^\top y)^2(( S^{1/2}\nabla F( \theta)\bv)^\top y)^2\Big]}{(1-\epsilon)^2\Big(\big( S^{1/2} r( \theta)\big)^\top S^{1/2} r( \theta)\Big)\Big(\big( S^{1/2}\nabla F( \theta)\bv\big)^\top\big( S^{1/2}\nabla F( \theta)\bv\big)\Big)}
    \\=&\frac{\epsilon}{(1-\epsilon)^2}+\frac{\bbE\left[(r(\theta)^\top x)^2\left(\bv^\top\nabla F(\theta)^\top xx^\top \nabla F(\theta)\bv\right)\right]}{(1-\epsilon)^2\left(r(\theta)^\top S r(\theta)\right)\left(\bv^\top\nabla F( \theta)^\top S\nabla F(\theta)\bv\right)}
    \\\overset{\eqref{equ: proof: strong align: def population}}{=}&\frac{\epsilon}{(1-\epsilon)^2}+\frac{\bv^\top\bar{\Sigma}_1(\theta)\bv}{2(1-\epsilon)^2\bar{\cL}(\theta)\bv^\top \bar{G}(\theta)\bv}
    \overset{\text{Lemma \ref{lemma: OLM equation Gaussian}}}{=}\frac{\epsilon}{(1-\epsilon)^2}+\frac{2\bar{\cL}(\theta)\bv^\top \bar{G}(\theta)\bv+\big(\nabla\bar{\cL}( \theta)^\top\bv\big)^2}{2(1-\epsilon)^2\bar{\cL}( \theta)\bv^\top\bar{G}(\theta)\bv}
    \\=&\frac{1+\epsilon}{(1-\epsilon)^2}+\frac{\big(\nabla\bar{\cL}( \theta)^\top\bv\big)^2}{2(1-\epsilon)^2\bar{\cL}( \theta)\bv^\top\bar{G}( \theta)\bv}\overset{\text{Lemma \ref{lemma: OLM nabla-L-v bound}}}{\leq}\frac{1+\epsilon}{(1-\epsilon)^2}+\frac{1}{(1-\epsilon)^2}=\frac{2+\epsilon}{(1-\epsilon)^2}.
\end{align*}

Moreover, if $\<\bv,\cL( \theta)\>=0$, then the bound is
\begin{align*}
    {\rm RHS}\leq\frac{1+\epsilon}{(1-\epsilon)^2}.
\end{align*}

In the similar way, we can derive the lower bound for~\eqref{equ of proof: thm: strong: arb dir:: emp-pop}:
\begin{align*}
    {\rm LHS}=&\frac{\bv^\top\bar{\Sigma}_1(\theta)\bv}{2(1+\epsilon)^2\bar{\cL}(\theta)\bv^\top \bar{G}( \theta)\bv}-\frac{\epsilon}{(1+\epsilon)^2}
    \overset{\text{Lemma \ref{lemma: OLM equation Gaussian}}}{=}\frac{2\bar{\cL}(\theta)\bv^\top \bar{G}(\theta)\bv+\big(\nabla\bar{\cL}( \theta)^\top\bv\big)^2}{2(1+\epsilon)^2\bar{\cL}( \theta)\bv^\top \bar{G}(\theta)\bv}-\frac{\epsilon}{(1+\epsilon)^2}
    \\\geq&\frac{1}{(1+\epsilon)^2}-\frac{\epsilon}{(1+\epsilon)^2}=\frac{1-\epsilon}{(1+\epsilon)^2}.
\end{align*}
So for any $ S^{1/2} u( \theta)\ne 0, S^{1/2}\nabla F( \theta)\bv\ne0$, we have
\[
\frac{1-\epsilon}{(1+\epsilon)^2}\leq g( \theta;\bv)\leq\frac{2+\epsilon}{(1-\epsilon)^2}.
\]
% Moreover, if $\<\bv,\nabla\cL( \theta)\>=0$, then
% \[
% \frac{1-\epsilon}{(1+\epsilon)^2}\leq
% g( \theta;\bv)\leq\frac{1+\epsilon}{(1-\epsilon)^2}.
% \]

Hence, we have proved this theorem:
For any $\epsilon,\delta>0$, if 
\[
n\gtrsim\max\bbracket{\sbracket{d^2\log^2\sbracket{1/{\epsilon}}+\log^2(1/\delta)}/\epsilon,   \sbracket{d\log\sbracket{1/{\epsilon}}+\log(1/\delta)}/\epsilon^2}, 
\]
then \wp at least $1-\delta$, the strong alignment holds uniformly:
\begin{align*}
    \ \frac{1-\epsilon}{(1+\epsilon)^2}\leq\inf_{ \theta,\bv\in\bbR^p}g( \theta;\bv)\leq\sup_{ \theta,\bv\in\bbR^p}g( \theta;\bv)\leq\frac{2+\epsilon}{(1-\epsilon)^2},
    % \\
    % \text{(ii).}&\ \frac{1-\epsilon}{(1+\epsilon)^2}\leq\inf_{ \theta\in\bbR^p,\<\bv,\nabla\cL( \theta)\>=0}g( \theta;\bv)\leq\sup_{ \theta\in\bbR^p,\<\bv,\nabla\cL( \theta)\>=0}g( \theta;\bv)\leq\frac{1+\epsilon}{(1-\epsilon)^2}.
\end{align*}
\qed

\subsection{Proof of Theorem~\ref{thm: strong: LM}}

With the preparation of 
Lemma~\ref{lemma: quadratic covarience concentrate} and Lemma~\ref{lemma: strong: 1-fix 1-arb bound}, now we give the proof of Theorem~\ref{thm: strong: LM}.

We follow the proof of Theorem~\ref{thm: strong: OLM}. 
Notice that for linear model,
\begin{align*}
    \tilde{w}(\theta,v)=\frac{ S^{1/2}\nabla F( \theta)\bv}{\left\| S^{1/2}\nabla F( \theta)\bv\right\|}=\frac{ S^{1/2}\bv}{\left\|S^{1/2}\bv\right\|},
\end{align*}
independent of $\theta$. Thus, we denote it as $\tilde{w}(v):=\frac{ S^{1/2}\bv}{\left\|S^{1/2}\bv\right\|}$. Then,
\begin{align*}
    g(\theta;\bv)=\frac{\frac{1}{n}\sum\limits_{i=1}^n\Big(\tilde{r}(\theta)^\top z_i\Big)^2\Big(\tilde{ w}(\bv)^\top z_i\Big)^2}{\frac{1}{n}\sum\limits_{i=1}^n\Big(\tilde{ r}( \theta)^\top z_i\Big)^2\cdot\frac{1}{n}\sum\limits_{i=1}^n\Big(\tilde{w}(\bv)^\top z_i\Big)^2}.
\end{align*}

By Lemma~\ref{lemma: quadratic covarience concentrate}, for any $\epsilon\in(0,1)$, select $\delta=\epsilon^d$, if we choose $n\gtrsim\sbracket{d\log\sbracket{1/\epsilon}+\log(1/\delta)}/\epsilon^2\gtrsim d\log\sbracket{1/\epsilon}/\epsilon^2$,
then \wp at least $1-\delta=1-\epsilon^d$, the following inequalities hold:
\begin{gather*}
\sup\limits_{\theta\in\bbS^{d-1}}\left|\frac{1}{n}\sum\limits_{i=1}^n\Big(\tilde{\theta}(\bv)^\top z_i\Big)^2-1\right|\leq\epsilon,
\quad\sup\limits_{\bv\in\bbS^{d-1}}\left|\frac{1}{n}\sum\limits_{i=1}^n\Big(\tilde{w}(\bv)^\top z_i\Big)^2-1\right|\leq\epsilon.
\end{gather*}

Then we consider the estimate about the numerator of $g(\theta,v)$ for $K$ fixed direction $v\in\cV=\{v_1,\cdots,v_K\}$. 
By Lemma~\ref{lemma: strong: 1-fix 1-arb bound}, for any $0<\epsilon<1/2$, there are constants $C_1=C_1(\epsilon)>0$ and $C_2=C_2(\epsilon)>0$, such that if $n\geq C_1 d\log d$, then with probability at least $1-\frac{C_2}{n^2}$, it holds

\begin{align*}
    &\bbP\left(\sup_{v\in\cV}\sup_{\theta\in\bbS^{d-1}}\left|\frac{1}{n}\sum\limits_{i=1}^n\Big(\tilde{r}(\theta)^\top z_i\Big)^2\Big(\tilde{w}(\bv)^\top z_i\Big)^2-\bbE\left[\Big(\tilde{r}(\theta)^\top z_1\Big)^2\Big(\tilde{ w}(\bv)^\top z_1\Big)^2\right]\right|\geq\epsilon\right)
    \\\leq&\sum_{k=1}^K\bbP\left(\sup_{\theta\in\bbS^{d-1}}\left|\frac{1}{n}\sum\limits_{i=1}^n\Big(\tilde{r}(\theta)^\top z_i\Big)^2\Big(\tilde{w}(\bv_k)^\top z_i\Big)^2-\bbE\left[\Big(\tilde{r}(\theta)^\top z_1\Big)^2\Big(\tilde{ w}(\bv_k)^\top z_1\Big)^2\right]\right|\geq\epsilon\right)
    \\\leq&\frac{KC_2}{n^2}.
\end{align*}

Now we combine these two bounds. For any $\epsilon\in(0,1/2)$, there exist $\tilde{C}_1(\epsilon):=\max\{C(\epsilon),\log(1/\epsilon)/\epsilon^2\}>0$ and $C_2=C_2(\epsilon)>0$, such that: if we choose $n\gtrsim \tilde{C}_1 d\log d$, then \wp at least $1-\frac{C_2}{n^2}-\epsilon^d$,
\begin{gather*}
    \sup\limits_{\theta\in\bbS^{d-1}}\left|\frac{1}{n}\sum\limits_{i=1}^n\Big(\tilde{\theta}(\bv)^\top z_i\Big)^2-1\right|\leq\epsilon,
    \quad\sup\limits_{\bv\in\bbS^{d-1}}\left|\frac{1}{n}\sum\limits_{i=1}^n\Big(\tilde{w}(\bv)^\top z_i\Big)^2-1\right|\leq\epsilon,
    \\
    \sup_{v\in\cV}\sup_{\theta\in\bbS^{d-1}}\left|\frac{1}{n}\sum\limits_{i=1}^n\Big(\tilde{r}(\theta)^\top z_i\Big)^2\Big(\tilde{w}(\bv)^\top z_i\Big)^2-\bbE\left[\Big(\tilde{r}(\theta)^\top z_1\Big)^2\Big(\tilde{ w}(\bv)^\top z_1\Big)^2\right]\right|\geq\epsilon.
\end{gather*}

Therefore, for any $\theta\in\bbR^p$ and $v\in\cV$, 
\begin{align*}
    \frac{\bbE\Big[(\tilde{ r}(\theta)^\top z_1)^2(\tilde{w}( \bv)^\top z_1)^2\Big]-\epsilon}{(1+\epsilon)^2}\leq g( \theta;\bv)\leq\frac{\bbE\Big[(\tilde{r}(\theta)^\top z_1)^2(\tilde{ w}(\bv)^\top z_1)^2\Big]+\epsilon}{(1-\epsilon)^2}.
\end{align*}
Then in the same way as the proof of Lemma~\ref{thm: strong: OLM}, it holds that
\begin{align*}
\frac{1-\epsilon}{(1+\epsilon)^2}\leq g( \theta;\bv)\leq\frac{2+\epsilon}{(1-\epsilon)^2}.
\end{align*}

Thus, we complete the proof.

\qed

\section{Proofs in Section \ref{section: escaping}: Escape direction of SGD}
\label{appendix: proof: escaping}

\subsection{Proof of Theorem \ref{thm: escape: SGD}}
Recall that $ w(t)=\sum_{i=1}^d w_i(t) u_i$ with  $w_i(t)= u_i^{\top} w(t)$.
Then, $w_i(t+1)=(1-\eta\lambda_i)w_i(t)+\eta \xi(t)^\top u_i$. Taking the expectation of the  square of both sides, we obtain
\begin{align*}
     \bbE\big[w_i^2(t+1)\big]=(1-\eta\lambda_i)^2\bbE\big[w_i^2(t)\big]+\eta^2\EE[| u_i^{\top} \xi(t)|^2],
\end{align*}
According to Assumption~\ref{assumption: eigen-alignment}, there exists $A_1,A_2>0$ such that for any $i\in[d]$, 
\begin{align*}
    A_1\lambda_i \cL( w_t)\leq \EE[| u_i^T \xi(t)|]\leq A_2\lambda_i \cL( w_t).
\end{align*}

Let $X_t=\sum_{i=1}^k\lambda_i \EE[w_i^2(t)], Y_t = \sum_{i=k+1}^d \lambda_i \EE[w_i^2(t)]$  denote the components of loss energy along sharp and flat directions, respectively. And we denote $D_k(t):=Y_t/X_t$.

Plugging the fact that $2\cL( w(t))=X_t+Y_t$ into the  two formulations above, we can obtain the following component dynamics:
\begin{equation}\label{equ: proof: escape: estimate X Y}
\begin{aligned}
    X_{t+1}&\leq \alpha_k X_t + A_2 \eta^2 (\sum_{i=1}^k\lambda_i^2)(X_t+Y_t),\\
    X_{t+1}&\geq  A_1 \eta^2 (\sum_{i=1}^k\lambda_i^2)(X_t+Y_t),\\ 
    Y_{t+1}&\geq A_1 \eta^2\big(\sum_{i=k+1}^d\lambda_i^2\big) (X_t+Y_t),
\end{aligned}
\end{equation}
where $\alpha_k\leq\max_{i=1,\dots,k} |1-\eta\lambda_i|^2$. The terms $\alpha_k X_t$ and $\beta_k Y_t$ capture the impact of the gradient, while the remaining terms originate from the noise.

From~\eqref{equ: proof: escape: estimate X Y}, we have the following estimate about $D_k(t+1)$:
\begin{equation}\label{equ: proof: escape: general lower bound D_k}
\begin{aligned}
    &D_k(t+1)=\frac{Y_{t+1}}{X_{t+1}}\geq\frac{A_1\eta^2\big(\sum_{i=k+1}^d\lambda_i^2\big) (X_t+Y_t)}{\alpha_k X_t + A_2 \eta^2 (\sum_{i=1}^k\lambda_i^2)(X_t+Y_t)}
    \\=&\frac{A_1\sum_{i=k+1}^d\lambda_i^2}{A_2\sum_{i=1}^k\lambda_i^2}\cdot\frac{1}{1+\frac{\alpha_k}{A_2\eta^2\sum_{i=k+1}^d\lambda_i^2}\frac{X_t}{X_t+Y_t}}
    \\\geq&\frac{A_1\sum_{i=k+1}^d\lambda_i^2}{A_2\sum_{i=1}^k\lambda_i^2}\cdot\frac{1}{1+\frac{\max\limits_{1\leq i\leq k}|1-\eta\lambda_i|^2}{A_2\eta^2\sum_{i=1}^k\lambda_i^2}\frac{X_t}{X_t+Y_t}}.
\end{aligned}
\end{equation}

We will prove this theorem for the learning rate $\eta=\frac{\beta}{\norm{G( \theta^*)}_F}$, where $\beta\geq\frac{1.1}{\sqrt{A_1}}$.

\paragraph{Case (I). Small learning rate $\eta\in[\frac{1.1}{\sqrt{A_1}\norm{G( \theta^*)}_{\rF}},\frac{1}{\lambda_1}]$.}

In this step, we consider $\eta=\frac{\beta}{\norm{G( \theta^*)}_F}$ such that $\beta\geq\frac{1.1}{\sqrt{A_1}}$ and $\eta\leq\frac{1}{\lambda_1}$. Then we have:
\begin{align*}
    \frac{\max\limits_{1\leq i\leq k}|1-\eta\lambda_i|^2}{A_2\eta^2\sum_{i=k+1}^d\lambda_i^2}
    \leq\frac{1}{A_2\eta^2\sum_{i=1}^k\lambda_i^2}.
\end{align*}
Notice that~\eqref{equ: proof: escape: estimate X Y} also ensures:
\begin{align*}
    (X_{t+1}+Y_{t+1})\geq A_1\eta^2\big(\sum_{i=1}^d\lambda_i^2\big)(X_t+Y_t).
\end{align*}
Combining this inequality with~\eqref{equ: proof: escape: estimate X Y}, we have the estimate:
\begin{align*}
    &\frac{X_{t+1}}{X_{t+1}+Y_{t+1}}
    \leq\frac{\alpha_k X_t + A_2 \eta^2 (\sum_{i=1}^k\lambda_i^2)(X_t+Y_t)}{X_{t+1}+Y_{t+1}}
    \\\leq&\frac{\alpha_k X_t}{A_1\eta^2\big(\sum_{i=1}^d\lambda_i^2\big)(X_{t}+Y_{t})}+\frac{A_2(\sum_{i=1}^k\lambda_i^2)}{A_1\big(\sum_{i=1}^d\lambda_i^2\big)}
\end{align*}

For simplicity, we denote $W_t:=\frac{X_t}{X_t+Y_t}$, $A:=\frac{\alpha_k}{A_1\eta^2\big(\sum_{i=1}^d\lambda_i^2\big)}$, and $B:=\frac{A_2(\sum_{i=1}^k\lambda_i^2)}{A_1\big(\sum_{i=1}^d\lambda_i^2\big)}$. 

From $\eta\leq1/3$, we have $\alpha_k\leq1$ and $A\leq\frac{1}{A_1\eta^2\big(\sum_{i=1}^d\lambda_i^2\big)}=\frac{1}{A_1\beta^2}<1$. 
Moreover, it holds that
\begin{align*}
    W_{t+1}\leq& A W_t+B
    \leq A(AW_{t-1}+B)+B=A^2 W_{t-1}+B(1+A)
    \\\leq&\cdots\leq A^{t+1}W_0+B(1+A+\cdots+A^t)=A^{t+1}W_0+\frac{1-A^{t+1}}{1-A}B
\end{align*}
On the one hand, if we choose
\begin{align*}
t\geq\frac{\log\sbracket{1/W_0A_2\eta^2\sum_{i=1}^k\lambda_i^2}}{\log\sbracket{A_1\beta^2}},
\end{align*}
then we have
\begin{align*}
    A^{t}W_0\leq\sbracket{\frac{\alpha_k}{A_1\eta^2(\sum_{i=1}^d\lambda_i^2)}}^{t}W_0{\leq}\sbracket{\frac{1}{A_1\beta^2}}^{t}W_0\leq A_2\eta^2\sum_{i=1}^k\lambda_i^2.
\end{align*}

On the other hand, if we choose $t\geq1$,
then it holds that
\begin{align*}
    \frac{1-A^{t}}{1-A}B\leq B=\frac{A_2(\sum_{i=1}^k\lambda_i^2)}{A_1\big(\sum_{i=1}^d\lambda_i^2\big)}\leq A_2\eta^2\sum_{i=1}^k\lambda_i^2.
\end{align*}

Hence, if we choose
\begin{align*}
t\geq\max\bbracket{1,\frac{\log\sbracket{1/W_0A_2\eta^2\sum_{i=1}^k\lambda_i^2}}{\log\sbracket{A_1\beta^2}}},
\end{align*}
then we have
\begin{align*}
    \frac{X_t}{X_t+Y_t}=W_t\leq A^t W_0+\frac{1-A^{t}}{1-A}B\leq2A_2\eta^2\sum_{i=1}^k\lambda_i^2,
\end{align*}
 
which implies that
\begin{align*}
    &\text{RHS of~\eqref{equ: proof: escape: general lower bound D_k}}\geq\frac{A_1\sum_{i=k+1}^d\lambda_i^2}{A_2\sum_{i=1}^k\lambda_i^2}\cdot\frac{1}{1+\frac{\max\limits_{1\leq i\leq k}|1-\eta\lambda_i|^2}{A_2\eta^2\sum_{i=1}^k\lambda_i^2}\frac{X_t}{X_t+Y_t}}
    \\\geq&
    \frac{A_1\sum_{i=k+1}^d\lambda_i^2}{A_2\sum_{i=1}^k\lambda_i^2}\cdot\frac{1}{1+\frac{1}{A_2\eta^2\sum_{i=1}^k\lambda_i^2}\cdot 2A_2\eta^2\sum_{i=1}^k\lambda_i^2}=\frac{A_1\sum_{i=k+1}^d\lambda_i^2}{3A_2\sum_{i=1}^k\lambda_i^2}.
\end{align*}

\paragraph{Case (II). Large learning rate $\eta\geq1/\lambda_1$.}

In this step, we consider $\eta\geq\frac{1}{\lambda_1}$. Then for any $t\geq0$, we have:

\begin{align*}
    &\text{RHS of~\eqref{equ: proof: escape: general lower bound D_k}}=\frac{A_1\sum_{i=k+1}^d\lambda_i^2}{A_2\sum_{i=1}^k\lambda_i^2}\cdot\frac{1}{1+\frac{\alpha_k}{\sum_{i=k+1}^d\lambda_i^2}\frac{X_t}{X_t+Y_t}}
    \geq\frac{A_1\sum_{i=k+1}^d\lambda_i^2}{A_2\sum_{i=1}^k\lambda_i^2}\cdot\frac{1}{1+\frac{\max\limits_{i\in[k]} |1-\eta\lambda_i|^2}{A_2\eta^2\sum_{i=1}^k\lambda_i^2}}
    \\\geq&\frac{A_1\sum_{i=k+1}^d\lambda_i^2}{A_2\sum_{i=1}^k\lambda_i^2}\cdot\frac{1}{1+\frac{\max\bbracket{1,|1-\eta\lambda_1|^2}}{A_2\eta^2\sum_{i=1}^k\lambda_i^2}}\geq\frac{A_1\sum_{i=k+1}^d\lambda_i^2}{A_2\sum_{i=1}^k\lambda_i^2}\cdot\frac{1}{1+\frac{1}{A_2}}=\frac{A_1\sum_{i=k+1}^d\lambda_i^2}{(A_2+1)\sum_{i=1}^k\lambda_i^2}.
\end{align*}

Combining Case (I) and (II), we obtain this theorem: 
If we choose the learning rate $\eta=\frac{\beta}{\norm{G( \theta)}_F}$, where $\beta\geq\frac{1.1}{\sqrt{A_1}}$, 
then for any
\begin{align*}
t\geq\max\bbracket{1,\frac{\log\sbracket{1/W_0A_2\eta^2\sum_{i=1}^k\lambda_i^2}}{\log\sbracket{A_1\beta^2}}},
\end{align*}
we have
\begin{align*}
    D_k(t+1)\geq\frac{A_1\sum_{i=k+1}^d\lambda_i^2}{\max\{3A_2,A_2+1\}\sum_{i=1}^k\lambda_i^2}.
\end{align*}
\qed

\subsection{Proof of Proposition \ref{thm: escape GD}}
Recall that $ w(t)=\sum_{i=1}^d w_i(t) u_i$ with  $w_i(t)= u_i^{\top} w(t)$.
Then, for GD, $w_i(t+1)=(1-\eta\lambda_i)w_i(t)$, which implies:
\begin{align*}
    w_i(t)=(1-\eta\lambda_i)^t w_i(0).
\end{align*}
Therefore, for $\eta=\beta/\lambda_1$ ($\beta>2$), it holds that
\begin{align*}
    D_1(t)=\frac{\sum_{i=2}^d\lambda_i w_i^2(t)}{\lambda_1 w_1^2(t)}=\frac{\sum_{i=2}^d \lambda_i (1-\eta\lambda_i)^{2t}w_i^2(0)}{\lambda_1 (1-\eta\lambda_1)^{2t}w_1^2(0)}.
\end{align*}
\qed

\section{Useful Inequalities}

\begin{lemma}[Bernstein's Inequality \citep{vershynin2018high}]\label{lemma: bernstien}
Suppose $\{X_1,\cdots,X_n\}$ are independent sub-Exponential random variables with $\left\|X_i\right\|_{\psi_1}\leq K$. Then there exists an absolute constant $c>0$ such that for any $t\geq0$, we have:
\[
\mathbb{P}\left(\left|\frac{1}{n}\sum_{i=1}^n X_i-\frac{1}{n}\sum_{i=1}^n\bbE\big[X_i\big]\right|>t\right)\leq2\exp\left(-cn\min\Big\{\frac{t}{K},\frac{t^2}{K^2}\Big\}\right).
\]

\end{lemma}

\begin{lemma}[Hanson-Wright's Inequality \citep{vershynin2018high}]\label{lemma: hanson-wright}
Let $   X=(X_1,\cdots,X_n)\in\bbR^{n}$ be a random vector with independent mean zero sub-Gaussian coordinates. Let $ A$ be an $n\times n$ matrix. Then, there exists an absolute constant $c$ such that for every $t\geq0$, we have
\begin{align*}
    \bbP\sbracket{\left|   X^\top A   X-\bbE[   X^\top A   X]\right|\geq t}\leq2\exp\sbracket{-c\min\bbracket{\frac{t^2}{K^4\norm{ A}_\rF^2},\frac{t}{K^2\norm{ A}_2}}},
\end{align*}
where $K=\max_{i}\norm{X_i}_{\psi_2}$.
\end{lemma}

\begin{lemma}[Covariance Estimate for sub-Gaussian Distribution \citep{vershynin2018high}]\label{lemma: covariance estimate sub Gaussian}
Let $\bx,\bx_1,\cdots,\bx_n$ be i.i.d. random vectors in $\mathbb{R}^d$. More precisely, assume that there exists $K\geq1$ s.t. $\left\|\left<\bx,\bv\right>\right\|_{\psi_2}\leq K\left\|\left<\bx,\bv\right>\right\|_{L_2}\text{ for any } \bv\in\bbS^{d-1}$,
Then for any $u\geq0$, \wp at least $1-2\exp(-u)$ one has
\begin{align*}
    \left\|\frac{1}{n}\sum_{i=1}^n\bx_i\bx_i^\top-\bbE\big[\bx\bx^\top\big]\right\|\leq CK^2\left(\sqrt{\frac{d+u}{n}}+\frac{d+u}{n}\right)\left\|\bbE\big[\bx\bx^\top\big]\right\|,
\end{align*}
where $C$ is an absolute positive constant.

\end{lemma}

\begin{definition}[Sub-Weibull Distribution]\label{def: sub-Weibull} 
We define $X$ as a sub-Weibull random variable if it has a
bounded $\psi_\beta$-norm. The $\psi_\beta$-norm of $X$ for any $\beta>0$ is defined as
\[
\left\|X\right\|_{\psi_\beta}:=\inf\Big\{C>0:\bbE\big[\exp(|X|^\beta/C^\beta)\big]\leq2\Big\}.
\]
Particularly, when $\beta=1$ or $2$, sub-Weibull random variables reduce to sub-Exponential or sub-Gaussian random variables, respectively.
\end{definition}

\begin{lemma}[Concentration Inequality for Sub-Weibull Distribution, Theorem 3.1 in \citep{hao2019bootstrapping}]\label{lemma: concentration SW} Suppose $\{X_i\}_{i=1}^n$ are independent sub-Weibull random variables with $\left\|X_i\right\|_{\psi_\beta}\leq K$. Then there exists an absolute constant $C_\beta$
only depending on $\beta$ such that for any $\delta\in(0,1/e^2)$, \wp at least $1-\delta$, we have
\[
\left|\frac{1}{n}\sum_{i=1}^n X_i-\frac{1}{n}\sum_{i=1}^n\bbE\big[X_i\big]\right|\leq C_\beta K
\Bigg(\Big(\frac{\log(1/\delta)}{n}\Big)^{1/2}+\frac{\big(\log(1/\delta)\big)^{1/\beta}}{n}\Bigg).
\]

\end{lemma}

% \begin{lemma}[Proposition 1.1 in \citep{gotze2021concentration}]\label{lemma: hanson-wright-type sub-E}
% Let $   X=(X_1,\cdots,X_n)\in\bbR^{n}$ be a random vector satisfying $\bbE X_i=0,\bbE X_i^2=\sigma_i^2,\norm{X_i}_{\psi_i}\leq K$ for some $\alpha\in(0,1]\cup{2}$, and $ A$ be a symmetric $n\times n$ matrix. Then, there exists an absolute constant $c$ such that for every $t\geq0$, we have
% \begin{align*}
%     \bbP\sbracket{\left|   X^\top A   X-\bbE[   X^\top A   X]\right|\geq t}\leq2\exp\sbracket{-c\min\bbracket{\frac{t^2}{K^4\norm{ A}_\rF^2},\sbracket{\frac{t}{K^2\norm{ A}_2}}^{\frac{\alpha}{2}}}}.
% \end{align*}
% \end{lemma}

% \begin{lemma}[Covariance Estimate for sub-Exponential Distribution, Theorem 3.1 in \citep{adamczak2011sharp}]\label{lemma: covariance estimate sub Exp}\ \\
%  Let $d\leq n$ be positive integers and $\psi,K\geq1$. Let $\bx_1,\cdots,\bx_n$ be independent random vectors in $\mathbb{R}^d$ satisfying:
%  \begin{align*}
%     &\text{\rm (i). }\max_{i\in[n]}\sup\limits_{\bv\in\bbS^{d-1}}\left\|\left<\bx_i,\bv\right>\right\|_{\psi_1}\leq\psi;
%     \\&\text{\rm (ii). }
%     \mathbb{P}\left(\max_{i\in[n]}\left\|\bx_i\right\|/\sqrt{d}>K\max\{1,(n/d)^{1/4}\}\right)\leq\exp(-\sqrt{d}).
%  \end{align*}
% Then \wp at least $1-\exp(-c\sqrt{d})$ one has
% \[
% \sup_{\bv\in\bbS^{d-1}}\left|\frac{1}{n}\sum_{i=1}^n\Big((\bx_i^\top\bv)^2-\bbE\big[(\bx_i^\top\bv)^2\big]\Big)\right|\leq C(\psi+K)^2\sqrt{\frac{d}{n}},
% \]
% where $c,C$ are absolute positive constants.
% \end{lemma}

\begin{lemma}[Cauchy-Schwarz Inequalities]\label{lemma: Cauchy-Schwarz}\ \\
(1) Let $ S\in\mathbb{R}^{n\times n}$ be a positive symmetric definite matrix. For any $\bx, y\in\mathbb{R}^n$, we denote $\left<\bx, y\right>_{ S}:=\bx^\top S y$ and $\left\|\bx\right\|_{ S}:=\sqrt{\left<\bx,\bx\right>_{ S}}$, then we have $\left|\left<\bx, y\right>_{ S}\right|\leq\left\|\bx\right\|_{ S}\left\| y\right\|_{ S}$.
\\
(2) Given two random variables $X$ and $Y$, it holds that $\left|\bbE[XY]\right|\leq\sqrt{\bbE[X^2]}\sqrt{\bbE[Y^2]}$.
\end{lemma}

% \begin{lemma}\label{lemma: trace-stand-orth-basis}
% Let $\{\bv_1,\cdots,\bv_n\}$ be an standard orthogonal basis in $\mathbb{R}^n$. Then for any $\boldsymbol{A}\in\mathbb{R}^{n\times n}$, we have $\sum\limits_{i=1}^n\bv_i^\top\boldsymbol{A}\bv_i={\rm Tr}(\boldsymbol{A})$.
% \end{lemma}

\end{document}